\documentclass[11pt,twoside]{article}
\usepackage[margin=1in]{geometry}

\usepackage[utf8]{inputenc} 
\usepackage[T1]{fontenc}    
\usepackage{hyperref}       
\usepackage{url}            
\usepackage{booktabs}       
\usepackage{amsmath,amssymb,amsfonts,amsthm,graphicx}       
\usepackage{nicefrac}       
\usepackage{microtype}      
\usepackage{wrapfig}
\usepackage{verbatim}
\usepackage{float}
\usepackage{enumerate}
\usepackage{enumitem}
\usepackage{multirow, array}
\usepackage{authblk}
\usepackage{natbib}
\usepackage{dsfont}

\usepackage{color}

\usepackage{hyperref}
\usepackage{caption}
\usepackage{subcaption}

\usepackage{microtype}

\DeclareGraphicsExtensions{.eps}

\newcommand\nc\newcommand
\nc\bfa{{\boldsymbol a}}\nc\bfA{{\boldsymbol A}}\nc\cA{{\mathcal A}}
\nc\bfb{{\boldsymbol b}}\nc\bfB{{\boldsymbol B}}\nc\cB{{\mathcal B}}
\nc\bfc{{\boldsymbol c}}\nc\bfC{{\boldsymbol C}}\nc\cC{{\mathcal C}}
\nc\sC{{\mathscr C}}
\nc\bfd{{\boldsymbol d}}\nc\bfD{{\boldsymbol D}}\nc\cD{{\mathcal D}}
\nc\bfe{{\boldsymbol e}}\nc\bfE{{\boldsymbol E}}\nc\cE{{\mathcal E}}
\nc\bff{{\boldsymbol f}}\nc\bfF{{\boldsymbol F}}\nc\cF{{\mathcal F}}
\nc\bfg{{\boldsymbol g}}\nc\bfG{{\boldsymbol G}}\nc\cG{{\mathcal G}}
\nc\bfh{{\boldsymbol h}}\nc\bfH{{\boldsymbol H}}\nc\cH{{\mathcal H}}
\nc\bfi{{\boldsymbol i}}\nc\bfI{{\boldsymbol I}}\nc\cI{{\mathcal I}}
\nc\bfj{{\boldsymbol j}}\nc\bfJ{{\boldsymbol J}}\nc\cJ{{\mathcal J}}
\nc\bfk{{\boldsymbol k}}\nc\bfK{{\boldsymbol K}}\nc\cK{{\mathcal K}}
\nc\bfl{{\boldsymbol l}}\nc\bfL{{\boldsymbol L}}\nc\cL{{\mathcal L}}
\nc\bfm{{\boldsymbol m}}\nc\bfM{{\boldsymbol M}}\nc\sM{{\mathscr M}}\nc\cM{{\mathcal M}}
\nc\bfn{{\boldsymbol n}}\nc\bfN{{\boldsymbol N}}\nc\cN{{\mathcal N}}
\nc\bfo{{\boldsymbol o}}\nc\bfO{{\boldsymbol O}}\nc\cO{{\mathcal O}}
\nc\bfp{{\boldsymbol p}}\nc\bfP{{\boldsymbol P}}\nc\cP{{\mathcal P}}
\nc\bfq{{\boldsymbol q}}\nc\bfQ{{\boldsymbol Q}}\nc\cQ{{\mathcal Q}}
\nc\bfr{{\boldsymbol r}}\nc\bfR{{\boldsymbol R}}\nc\cR{{\mathcal R}}
\nc\bfs{{\boldsymbol s}}\nc\bfS{{\boldsymbol S}}\nc\cS{{\mathcal S}}
\nc\bft{{\boldsymbol t}}\nc\bfT{{\boldsymbol T}}\nc\cT{{\mathcal T}}
\nc\bfu{{\boldsymbol u}}\nc\bfU{{\boldsymbol U}}\nc\cU{{\mathcal U}}
\nc\bfv{{\boldsymbol v}}\nc\bfV{{\boldsymbol V}}\nc\cV{{\mathcal V}}
\nc\bfw{{\boldsymbol w}}\nc\bfW{{\boldsymbol W}}\nc\cW{{\mathcal W}}
\nc\bfx{{\boldsymbol x}}\nc\bfX{{\boldsymbol X}}\nc\cX{{\mathcal X}}
\nc\bfy{{\boldsymbol y}}\nc\bfY{{\boldsymbol Y}}\nc\cY{{\mathcal Y}}
\nc\bfz{{\boldsymbol z}}\nc\bfZ{{\boldsymbol Z}}\nc\cZ{{\mathcal Z}}

\nc\diff{{\mathrm d}}
\nc\e{{\mathrm e}}
\nc\calC{{\mathcal C}}

\DeclareMathOperator{\supp}{supp}

\newcommand{\remove}[1]{}

\newtheorem{conj}{Conjecture}
\newtheorem{theorem}{Theorem}
\newtheorem{claim}{Claim}
\newtheorem{remark}[conj]{Remark}

\newtheorem{lemma}{Lemma}

\newcommand{\f}[1]{\mathbf{#1}}
\newcommand{\bb}[1]{\mathbb{#1}}
\newcommand{\ca}[1]{\mathcal{#1}}

\newcommand{\s}[1]{\mathsf{#1}}

\def\DEBUG{true}

\ifdefined\DEBUG

  \def\rem#1{{\marginpar{\raggedright\scriptsize #1}}}
  \newcommand{\barnr}[1]{\rem{\textcolor{red}{$\bullet$ #1}}}
  \newcommand{\aryar}[1]{\rem{\textcolor{green}{$\bullet$ #1}}}
\else

  \newcommand{\barnr}[1]{}
  \newcommand{\aryar}[1]{}

\fi

\newcommand\reals{{\mathbb R}}

\allowdisplaybreaks

\usepackage{algorithm}
\usepackage{algorithmic}

\title{Recovery of Sparse Signals from a Mixture of Linear Samples}
\author{Arya Mazumdar}
\author{Soumyabrata Pal}
\affil{College of Information and Computer Sciences\\ University of Massachusetts Amherst   Amherst, MA 01003, USA\\ E-mail: \texttt{arya@cs.umass.edu}, \texttt{soumyabratap@umass.edu}.}
\date{}

\begin{document}

\maketitle

\begin{abstract}
Mixture of linear regressions is a popular learning theoretic model that is used widely to represent heterogeneous data. In the simplest form, this model assumes that the labels are generated from either of two different linear models and mixed together. Recent works of Yin et al. and Krishnamurthy et al., 2019, focus on an experimental design setting of model recovery for this problem. It is assumed that the features can be designed and queried with to obtain their label. When queried, an oracle randomly selects one of the two different sparse linear models and generates a label accordingly. How many such oracle queries are needed to recover both of the models simultaneously? This question can also be thought of as a generalization of the well-known compressed sensing problem (Cand\`es and Tao, 2005, Donoho, 2006).
In this work we address this query complexity problem and provide efficient algorithms that improves on the previously best known results. 
\end{abstract}

\section{Introduction}

Suppose, there are two unknown distinct vectors $\f{\beta}^{1},\f{\beta}^{2} \in \mathbb{R}^{n}$, that we want to recover. We can measure these vectors by taking linear samples, however the linear samples come without the identifier of the vectors. To make this statement rigorous, assume
 the presence of an oracle $\mathcal{O}:\mathbb{R}^n \rightarrow \mathbb{R}$ which, when queried with a vector $\bfx \in \mathbb{R}^{n}$, returns the noisy output $y \in \mathbb{R}$:
 \vspace{-0.1in}
\begin{align}\label{eq:sample}
y= \langle \mathbf{x},\f{\beta} \rangle + \zeta
\end{align}
where $\f{\beta}$ is chosen uniformly from $\{
\f{\beta}^{1},\f{\beta}^{2}
\}$ and $\zeta$ is additive Gaussian noise with zero mean and known variance $\sigma^2>0$. We will refer to the values returned by the oracle given these queries as \emph{samples.}


 For a $\f{\beta}  \in\reals^n$, the best $k$-sparse approximation $\f{\beta}_{(k)}$  is defined to be the vector obtained from $\f{\beta}$ where all except the $k$-largest (by absolute value) coordinates are set to $0$. For each $\f{\beta} \in \{\f{\beta}^1,\f{\beta}^2\}$, our objective in this setting is to return a {\em sparse approximation} of $\hat{\f{\beta}}$ using minimum number of queries such that
\begin{align*}
||\hat{\f{\beta}}-\f{\beta}|| \le c||\f{\beta}-\f{\beta}_{(k)}||+\gamma
\end{align*}
where $c$ is an absolute constant, $\gamma$ is a user defined nonnegative parameter representing the precision up to which we want to recover the unknown vectors, and the norms are arbitrary. For any algorithm that performs this task, the total number of samples acquired from the oracle is referred to as the {\em query complexity.} 

If we had one, instead of two unknown vectors, then the problem would exactly be that of {\em compressed sensing} \cite{candes2006robust,donoho2006compressed}. However having two vectors makes this problem significantly different and challenging.  Further, if we allow  $\gamma = \Omega( \|\beta^1 -\beta^2\|)$, then we can treat all the samples to be coming from the same vector and output only a single vector as an approximation to both vectors. So in practice, obtaining  $\gamma = o( \|\beta^1 -\beta^2\|)$ is more interesting.

On another technical note, under this setting it is always possible to make the noise $\zeta$ negligible by increasing the norm of the query $\bfx$. To make the problem well-posed, let us define the Signal-to-Noise Ratio (SNR) for a query $\f{x}$:
\begin{align*}
\mathsf{SNR}(\f{x}) \triangleq \frac{\bb{E}|\langle \mathbf{x},  \f{\beta}^1-\f{\beta}^2  \rangle|^{2}}{\bb{E} \zeta^2}
\end{align*}
where the expectation is over the randomness of the query. Furthermore define the overall SNR to be $\mathsf{SNR}:= \max_{\f{x}} \mathsf{SNR}(\f{x})$, where the maximization is over all the queries used in the recovery process. 


\subsection{Most Relevant Works}
 Previous works that are  most relevant to our problem are by Yin et al.~\cite{yin2019learning} and Krishnamurthy et al.~\cite{KrishnamurthyM019}. Both of these papers address the exact same problem as above; but provide results under some restrictive conditions on the unknown vectors.
For example,  the results of \cite{yin2019learning} is valid only when,
\begin{itemize}
\item the unknown vectors are exactly $k$-sparse, i.e., has at most $k$ nonzero entries;
\item it must hold that,  
\begin{equation*}\label{eq:weirdassumption}
\beta^1_j \ne \beta^2_j \quad \mbox{ for each } \quad j \in \supp{\f{\beta}^1} \cap \supp{\f{\beta}^2} \ ,
\end{equation*} 
where $\beta_j$ denotes the $j$th coordinate of $\f{\beta}$, and $\supp{\f{\beta}}$ denotes the set of nonzero coordinates of $\f{\beta}$;
\item for some $\epsilon >0$ , $\f{\beta}^1, \f{\beta}^2 \in \{0, \pm \epsilon, \pm 2\epsilon, \pm 3\epsilon , \ldots\}^n$.
\end{itemize}

All of these assumptions, especially the later two, are severely restrictive. While the results of \cite{KrishnamurthyM019} are valid without the first two assumptions, they fail to get rid of the third, an assumption of the unknown vectors always taking discrete values. This is in particular unfavorable, because the resultant query/sample complexities (and hence the time complexity) in both the above papers has an exponential dependence on $\frac1\epsilon$. 


\subsection{Our Main Result}

In contrast to these earlier results, we provide a generic sample complexity result that does not require any of the assumptions used by the predecessor works. Our main result is following.

\begin{theorem}{\label{thm:main}}[Main Result] Let $\mathsf{NF}:= \frac{\gamma}{\sigma}$ (the  {\em noise factor}) where $\gamma>0$
is a parameter representing the desired recovery precision and $\sigma>0$ is the standard deviation of $\zeta$ in Eq.~\eqref{eq:sample}. 

Case 1. For any {$\gamma <  \left|\left|\beta^1-\beta^2\right|\right|_2/2$},
    there exists an algorithm that makes
\begin{align*}
O\Bigg(k\log n\log k\Big\lceil \frac{ \log k}{\log (\sqrt{\mathsf{SNR}}/\mathsf{NF})}\Big\rceil
 \Big\lceil\frac{1}{\mathsf{NF}^4\sqrt{\mathsf{SNR}}}+\frac{1}{\mathsf{NF}^2}\Big\rceil\Bigg)
\end{align*} 
 queries  to recover $\hat{\f{\beta}}^1,\hat{\f{\beta}}^2$, estimates of $\f{\beta}^1,\f{\beta}^2$, 
with high probability 
such that , for $i=1,2$,
\begin{align*}
||\hat{\f{\beta}}^i-\f{\beta}^{\pi(i)}||_2 \le \frac{c||\f{\beta}^{i}-\f{\beta}_{(k)}^{i}||_1}{\sqrt{k}}+O(\gamma) 
\end{align*}
where $\pi:\{1,2\}\rightarrow \{1,2\}$ is some permutation of $\{1,2\}$ and $c$ is a universal constant.

Case 2. For any {$\gamma = \Omega( \left|\left|\beta^1-\beta^2\right|\right|_2)$}, there exists an algorithm that makes
$O\Big(k\log n\Big\lceil \frac{\log k}{\s{SNR}}\Big\rceil\Big)$
 queries  to recover $\hat{\f{\beta}}$, estimates of both $\f{\beta}^1,\f{\beta}^2$, 
with high probability 
such that , for both $i=1,2$,
\begin{align*}
||\hat{\f{\beta}}-\f{\beta}^{i}||_2 \le \frac{c||\f{\beta}^{i}-\f{\beta}_{(k)}^{i}||_1}{\sqrt{k}}+O(\gamma) 
\end{align*}
where $c$ is a universal constant.

\end{theorem}
For a $\gamma =\Theta(\left|\left|\beta^1-\beta^2\right|\right|_2)$ the first case of the Theorem holds but using the second case may give better result in that regime of precision.
The second case of the theorem shows that if we allow a rather large precision error, then the number of queries is similar to the required number for recovering a single vector. This is expected, because in this case we can find just one line approximating both regressions.

The recovery guarantee that we are providing (an $\ell_2$-$\ell_1$ guarantee) is in line with the standard guarantees of the compressed sensing literature. In this paper, we are interested in the regime $\log n \le k \ll n$ as in compressed sensing. Note that, our number of required samples scales linearly with $k$ and has only poly-logarithmic scaling with $n$, and polynomial scaling with the noise $\sigma$. In the previous works~\cite{yin2019learning},~\cite{KrishnamurthyM019}, the complexities scaled exponentially with noise. 

Furthermore, the query complexity of our algorithm decreases with the Euclidean distance between the vectors (or the `gap') - which makes sense intuitively. Consider the case when when we want a precise recovery ($\gamma$ very small). It turns out that  when the gap is large, the query complexity varies as $O((\log \mathrm{gap})^{-1} )$ and when the gap is small the query complexity scale as $O((\mathrm{gap} \log \mathrm{gap})^{-1} )$.

\begin{remark}[The zero noise case] When $\sigma =0$, i.e., the samples are not noisy, the problem is still nontrivial, and is not covered by the statement of Theorem~\ref{thm:main}. However this case is strictly simpler to handle as it will involve only the alignment step (as will be discussed later), and not the mixture learning step. Recovery with $\gamma=0$ is possible with only $O(k \log n \log k)$ queries ({see Appendix \ref{sec:discuss} for a more detailed discussion on the noiseless setting}).
\end{remark}


\subsection{Other Relevant Works}
The problem we address can be seen as the active learning version of learning mixtures of linear regressions. Mixture of linear regressions is a natural synthesis of mixture models and  linear regression;
 a generalization of the basic linear regression problem of learning the best linear relationship between the  labels and the feature vectors. In this generalization, each label is stochastically generated by picking a linear relation uniformly from a set of two or more linear functions, evaluating this function on the features and possibly adding noise; the goal is to learn the set of  unknown linear functions.  
The problem has been studied at least for past three decades, staring with De Veaux \cite{de1989mixtures} with a recent surge of interest~\cite{chaganty2013spectral,faria2010fitting,stadler2010l,kwon2018global,viele2002modeling,yi2014alternating,yi2016solving}. 
In this literature a variety of algorithmic techniques to obtain polynomial sample complexity were proposed. To the best of our knowledge, St\"adler et al.~\cite{stadler2010l} were the first to impose sparsity on the solutions, where each linear function depends on only a small number of variables. However, many of the earlier papers on mixtures of linear regression, essentially consider the features to be fixed, i.e.,  part of the input, whereas recent works focus on the query-based model in the sparse setting, where features can be designed as queries~\cite{yin2019learning,KrishnamurthyM019}. The problem has numerous applications in modelling heterogeneous data arising in medical applications, behavioral health, and music perception \cite{yin2019learning}. 

This problem is a generalization of the  compressed sensing problem~\cite{candes2006robust,donoho2006compressed}. As a building block to our solution, we use results from exact parameter learning for Gaussian mixtures.
Both compressed sensing and learning mixtures of distributions \cite{dasgupta1999learning,titterington1985statistical} are immensely popular topics across statistics, signal processing and machine learning with a large body of prior work. We refer to an excellent survey by \cite{boche2015survey} for compressed sensing results (in particular the results of \cite{candes2008restricted} and \cite{baraniuk2008simple} are useful). For parameter learning in mixture models, we find the results of \cite{daskalakis2016ten,daskalakis2014faster,hardt2015tight,xu2016global,balakrishnan2017statistical,krishnamurthy20a} to be directly relevant. 

\subsection{Technical Contributions}
If the responses to the queries were to contain tags of the models they are coming from, then we could use rows of any standard compressed sensing matrix as queries and just segregate the responses using the tags. Then by running a compressed sensing recovery on the groups with same tags, we would be done. In what follows, we try to infer this `tag' information by making redundant queries.

If we repeat just the same query multiple time, the noisy responses are going to come from a mixture of Gaussians, with the actual responses being the component means. To learn the actual responses we rely on methods for parameter learning in Gaussian mixtures. It turns out that for different parameter regimes, different methods are best-suited for our purpose - and it is not known in advance what regime we would be in. The method of moments is a well-known procedure for parameter learning in Gaussian mixtures and rigorous theoretical guarantees on sample complexity exist~\cite{hardt2015tight}. However we are in a specialized regime of scalar uniform mixtures with known variance; and we leverage these information to get better sample complexity  guarantee for exact parameter learning (Theorem~\ref{thm:mom}). In particular we show that, in this case the mean and variance of the mixture are sufficient statistics to recover the unknown means, as opposed to the first six moments of the general case ~\cite{hardt2015tight}. While recovery using other methods (Algorithms~\ref{algo:em} and \ref{algo:single}) are straight forward adaption of known literature, we show that only a small set of samples are needed to determine what method to use.

It turns out that method of moments still needs significantly more samples than the other methods. However we can avoid using method of moments and use a less intensive method (such as EM, Algorithms~\ref{algo:em}), provided we are in a regime when the gap between the component means is high. The only fact is that the Euclidean distance between $\beta^1$ and $\beta^2$ are far does not guarantee that. However, if we choose the queries to be Gaussians, then the gap is indeed high with certain probability. If the queries were to be generated by any other distribution, then such fact will require strong anti-concentration inequalities that in general do not hold. Therefore, we cannot really work with any standard compressed sensing matrix, but have to choose Gaussian matrices (which are incidentally also good standard compressed sensing matrices).

The main technical challenge comes in the next step, alignment. For any two queries $\f{x},\f{x}',$ even if we know $y_1 = \langle \beta^1,\f{x}\rangle, y_2 = \langle \beta^2,\f{x}\rangle$ and $y_1'=\langle \beta^1,\f{x}'\rangle, y_2'= \langle \beta^2,\f{x}'\rangle$, we do not know how to club $y_1$ and $y_1'$ together as their order could be different. And this is an issue with all pairs of queries which leaves us with exponential number of possibilities to choose form. We form a simple error-correcting code to tackle this problem.

For two queries, $\f{x},\f{x}',$ we deal with this issue by designing two additional queries $\f{x}+\f{x}'$and  $\f{x}-\f{x}'.$ Now even if we mis-align, we can cross-verify with the samples from `sum' query and the `difference' query, and at least one of these will show inconsistency. We subsequently extend this idea to align all the samples. Once the samples are all aligned, we can just use some any recovery algorithm for compressed sensing to deduce the sparse vectors.

The rest of this paper is organized as follows. We give an overview of our algorithm in Sec.~\ref{sec:ov} , the actual algorithm is presented in Algorithm~\ref{algo:main}, which calls several subroutines. The process of denoising by Gaussian mixture learning is described in Sec.~\ref{sec:GMM}. The alignment problem is discussed in Sec.~\ref{sec:al} and the proof of Theorem ~\ref{thm:main} is wrapped up in Sec.~\ref{sec:tm}. Most proofs are delegated to the appendix in the supplementary material. Some `proof of concept' simulation results are also in the appendix.



\section{Main Results}

\subsection{Overview of Our Algorithm}\label{sec:ov}
Our scheme to recover the unknown vectors is described below. We will carefully chose the numbers $m,m'$ so that the overall query complexity meets the promise of Theorem~\ref{thm:main}.
\begin{itemize}[leftmargin=*,noitemsep,topsep=0em]
\item We  pick $m$ query vectors $\f{x}^1,\f{x}^2,\dots,\f{x}^m$ independently, each according to $\ca{N}(\f{0},\f{I}_n)$ where $\f{0}$ is the $n$-dimensional all zero vector and $\f{I}_n$ is the $n\times n$ identity matrix. 
\item (Mixture) We repeatedly query the oracle with $\f{x}^i$ for $T_i$ times for all $i \in [m]$ in order to offset the noise. The samples obtained from the repeated querying of $\f{x}^i$ is referred to as a \textit{batch} corresponding to $\f{x}^i$.  $T_i$ is referred to as the \textit{batchsize} of $\f{x}^i$. Our objective is to return $\hat{\mu}_{i,1}$ and $\hat{\mu}_{i,2}$, estimates of  $\langle \f{x}^i, \f{\beta}^1 \rangle$ and $\langle \f{x}^i, \f{\beta}^2 \rangle$ respectively from the batch of samples (details in Section \ref{sec:GMM}). However, it will not be possible to label which estimated mean corresponds to $\f{\beta}^1$ and which one corresponds to $\f{\beta}^2$. 
\item (Alignment) For some $m'<m$ and for each  $i \in [m], j \in [m']$ such that $i\neq j$, we  also  query the oracle with the vectors $\f{x}^i+\f{x}^j$ (\textit{sum query}) and $\f{x}^i-\f{x}^j$ (\textit{difference query}) repeatedly for $T_{i,j}^{\s{sum}}$ and $T_{i,j}^{\s{diff}}$ times respectively. Our objective is to cluster the set of estimated means $\{\hat{\mu}_{i,1},\hat{\mu}_{i,2}\}_{i=1}^{m}$ into two equally sized clusters such that all the elements in a particular cluster are good estimates of querying the same unknown vector.
\item Since the queries $\{\f{x}^i\}_{i=1}^{m}$ has the property of being a good compressed sensing matrix (they satisfy $\delta$-RIP condition, a sufficient condition for $\ell_2$-$\ell_1$ recovery in compressed sensing, with high probability), we can formulate a convex optimization problem using the estimates present in each cluster to recover the unknown vectors $\f{\beta}^1$ and $\f{\beta}^2$.
\end{itemize} 
It is evident that the sample (query) complexity will be $\sum_{i=1}^{m} T_i+\sum_{\substack{i \in [m], j \in [m']\\i \neq j}}T_{i,j}^{\s{sum}}+T_{i,j}^{\s{diff}}$. 
In the subsections below, we will show each step more formally and provide upper bounds on the sufficient batchsize for each query.

\subsection{Recovering Unknown Means from a Batch}{\label{sec:GMM}}
For a query $\f{x} \in \{\f{x}^1,\f{x}^2,\dots,\f{x}^m\}$, notice that the samples from the batch corresponding to $\f{x}$ is distributed according to a Gaussian mixture $\ca{M}$,
\begin{align*}
\ca{M} \triangleq \frac{1}{2}\ca{N}(\langle \f{x}, \f{\beta}^{1} \rangle,\sigma^2)+\frac{1}{2}\ca{N}(\langle \f{x}, \f{\beta}^{2} \rangle,\sigma^2),
\end{align*}
 an equally weighted mixture of two Gaussian distributions having means $\langle \f{x}, \f{\beta}^{1} \rangle,\langle \f{x}, \f{\beta}^{2} \rangle$  with known variance $\sigma^2$. For brevity, let us denote $\langle \f{x}, \f{\beta}^{1} \rangle$ by $\mu_1$ and $\langle \f{x}, \f{\beta}^{2} \rangle$ by $\mu_2$ from here on in this sub-section. In essence, our objective is to find the sufficient batchsize of $\f{x}$ so that it is possible to estimate $\langle \f{x}, \f{\beta}^{1} \rangle$ and $\langle \f{x}, \f{\beta}^{2} \rangle$ upto an additive error of $\gamma$. Below, we go over some methods providing theoretical guarantees on the sufficient sample complexity for approximating the means that will be suitable for different parameter regimes. 

\subsubsection{Recovery using EM algorithm}
\begin{algorithm}[htbp]
\caption{ \textsc{EM($\f{x},\sigma,T$)} Estimate the means $\langle \f{x},\f{\beta}^1 \rangle$, $\langle \f{x},\f{\beta}^2 \rangle$ for a query $\f{x}$ using EM algorithm \label{algo:em}} 
\begin{algorithmic}[1]
\REQUIRE An oracle $\ca{O}$ which when queried with a vector $\f{x} \in \bb{R}^n$ returns $\langle \f{x},\f{\beta} \rangle +\ca{N}(0,\sigma^2)$ where $\f{\beta}$ is sampled uniformly from $\{\f{\beta}^1,\f{\beta}^2\}$.
\FOR{$i=1,2,\dots,T$}
\STATE Query the oracle $\ca{O}$ with $\f{x}$ and obtain a response $y^i$.
\ENDFOR
\STATE  Set the function $w:\bb{R}^3 \rightarrow \bb{R}$ as $w(y,\mu_1,\mu_2)=e^{-(y-\mu_1)^2/2\sigma^2}\Big(e^{-(y-\mu_1)^2/2\sigma^2}+e^{-(y-\mu_2)^2/2\sigma^2}\Big)^{-1}.$
\STATE \textbf{Initialize} $\hat{\mu}_1^{0},\hat{\mu}_2^{0}$ randomly and $t=0$.
\WHILE{Until Convergence}
      \STATE $\hat{\mu}_1^{t+1}=\sum_{i=1}^{T}y_i w(y_i,\hat{\mu}^t_1,\hat{\mu}^t_2)/\sum_{i=1}^{T}w(y_i,\hat{\mu}^t_1,\hat{\mu}^t_2)$.
      \STATE $\hat{\mu}_2^{t+1}=\sum_{i=1}^{T}y_i w(y_i,\hat{\mu}^t_2,\hat{\mu}^t_1)/\sum_{i=1}^{T}w(y_i,\hat{\mu}^t_2,\hat{\mu}^t_1)$.
      \STATE $t \leftarrow t+1$.
\ENDWHILE
\STATE Return $\hat{\mu}_1^{t},\hat{\mu}_2^{t}$
\end{algorithmic}
\end{algorithm}
The Expectation Maximization (EM) algorithm  is widely known, and used for the purpose of parameter learning of Gaussian mixtures, cf.~\cite{balakrishnan2017statistical} and \cite{xu2016global}. The EM algorithm tailored towards recovering the parameters of the mixture $\ca{M}$ is described in Algorithm \ref{algo:em}. The following result can be derived from   \cite{daskalakis2016ten} (with our terminology) that gives a sample complexity guarantee of using EM algorithm. 
\begin{theorem}[Finite sample EM analysis \cite{daskalakis2016ten}]{\label{thm:EM}} From an equally weighted two component Gaussian mixture with unknown component means $\mu_1,\mu_2$ and known and shared variance
$\sigma^2$, a total $O\Big( \Big\lceil\sigma^6/(\epsilon^2(\mu_1-\mu_2)^4)\log 1/\eta\Big\rceil\Big)$ samples suffice to return $\hat{\mu}_1,\hat{\mu}_2$, such that for some permutation $\pi:\{1,2\} \rightarrow \{1,2\}$, for $i=1,2,$
\begin{align*}
 \left|\hat{\mu}_i-\mu_{\pi(i)} \right| \le \epsilon \;
\end{align*}
using the EM algorithm with probability at least $1-\eta$.
\end{theorem}
This theorem implies that EM algorithm requires smaller number of samples as the separation between the means $|\mu_1-\mu_2|$ grows larger. However, it is possible to have a better dependence on $|\mu_1-\mu_2|$, especially when it is small compared to $\sigma^2$.
\subsubsection{Method of Moments}
Consider any  Gaussian mixture with two components,
\begin{align*}
\ca{G} \triangleq p_1\ca{N}(\mu_1,\sigma_1^2)+p_2\ca{N}(\mu_2,\sigma_2^2),
\end{align*} 
where $0 < p_1,p_2 < 1$ and $p_1+p_2 =1.$
Define the variance of a random variable distributed according to $\ca{G}$ to be $$\sigma^2_{\ca{G}} \triangleq p_1p_2((\mu_1-\mu_2)^2+p_1\sigma_1^2+p_2\sigma_2^2.$$ It was shown in \cite{hardt2015tight}  that $\Theta(\sigma_{\ca{G}}^{12}/\epsilon^{12})$ samples are both necessary and sufficient to recover the unknown parameters $\mu_1,\mu_2,\sigma_1^2,\sigma_2^2$ upto an additive error of $\epsilon$. However, in our setting the components of the mixture $\ca{M}$ have the same known variance  $\sigma^2$ and further the mixture is equally weighted. Our first contribution is to show significantly better results for this special case.
\begin{theorem}{\label{thm:mom}}
With $O\Big( \Big\lceil \sigma^2_{\ca{M}}/\epsilon_1^2,\sigma^4_{\ca{M}}/\epsilon_2^2 \Big\rceil\log 1/\eta\Big)$ samples, Algorithm \ref{algo:mom}  returns $\hat{\mu}_1,\hat{\mu}_2$, such that for some permutation $\pi:\{1,2\} \rightarrow \{1,2\}$, we have, for $i=1,2,$
$
 \left|\hat{\mu}_i-\mu_{\pi(i)} \right| \le 2\epsilon_1+2\sqrt{\epsilon_2} \,
$
with probability at least $1-\eta$.
\end{theorem}   
This theorem states that $O(\sigma_{\ca{M}}^4)$ samples are sufficient to recover the unknown means of $\ca{M}$ (as compared to the $O(\sigma_{\ca{G}}^{12})$ result for the general case). This is because the mean and variance are sufficient statistics for this special case (as compared to first six excess moments in the general case). We first show two technical lemmas providing guarantees on recovering the mean and the variance of a random variable $X$ distributed according to $\ca{M}$. The procedure to return $\hat{M}_1$ and $\hat{M}_2$ (estimates of $\bb{E}X$ and ${\rm var}X$ respectively) is described in Algorithm \ref{algo:estimate}. 

\begin{lemma}{\label{lem:mean}}
$O\Big(\Big\lceil \sigma^2_{\ca{M}}/\epsilon_1^2 \Big\rceil\log \eta^{-1}\Big)$ samples divided into $B=36\log \eta^{-1}$ equally sized batches are sufficient to compute $\hat{M}_1$ (see Algorithm \ref{algo:estimate}) such that 
$\left|\hat{M}_1-\bb{E}X \right| \le \epsilon_1$ with probability at least $1-2\eta$. 
\end{lemma}
\begin{lemma}{\label{lem:var}}
$O\Big(\Big\lceil\sigma^4_{\ca{M}}/\epsilon_2^2 \Big\rceil\log \eta^{-1}\Big)$ samples divided into $B=36\log \eta^{-1}$ equally sized batches is sufficient to compute $\hat{M}_2$ (see Algorithm \ref{algo:estimate}) such that 
$\left|\hat{M}_2-{\rm var}X \right| \le \epsilon_2$ with probability at least $1-2\eta$.
\end{lemma}
The detailed proofs of Lemma \ref{lem:mean} and \ref{lem:var} can be found in Appendix \ref{sec:proof1}. We are now ready to prove Theorem \ref{thm:mom}.

\begin{algorithm}[htbp]
\caption{ \textsc{Estimate($\f{x},T,B$)} Estimating $\bb{E}X$ and ${\rm var}X$ for $X \sim \ca{M}$ \label{algo:estimate}}
\begin{algorithmic}[1]
\REQUIRE I.i.d samples $y^1,y^2,\dots,y^T \sim \ca{M}$ where $\ca{M}=\frac{1}{2}\ca{N}(\langle \f{x}, \f{\beta}^{1} \rangle,\sigma^2)+\frac{1}{2}\ca{N}(\langle \f{x}, \f{\beta}^{2} \rangle,\sigma^2)$.
\STATE Set $t=T/B$
\FOR{$i=1,2,\dots,B$}
\STATE Set Batch $i$ to be the samples $y^j$ for $j \in \{it+1,it+2,\dots,(i+1)t\}$.
\STATE Set $S^i_{1}=\sum_{j \in \text{ Batch }i} \frac{y^i}{t}$, $S^i_{2}=\sum_{j \in \text{ Batch }i} \frac{(y^i-S^i_1)^2}{t-1}$.
\ENDFOR
\STATE  $\hat{M}_1=\s{median}(\{S^i_1\}_{i=1}^{B})$,  $\hat{M}_2=\s{median}(\{S^i_2\}_{i=1}^{B})$. 
\STATE Return $\hat{M}_1,\hat{M}_2$.
\end{algorithmic}
\end{algorithm}
\begin{proof}[Proof of Theorem \ref{thm:mom}]
We will set up the following system of equations in the variables $\hat{\mu}_1$ and $\hat{\mu}_2$:
\begin{align*}
\hat{\mu}_1+\hat{\mu}_2 = 2\hat{M}_1 \; \text{and} \; (\hat{\mu}_1-\hat{\mu}_2)^2= 4\hat{M}_2-4\sigma^2 
\end{align*}
Recall that from Lemma \ref{lem:mean} and Lemma \ref{lem:var}, we have computed $\hat{M}_1$ and $\hat{M}_2$ with the following guarantees:
$
\left|\hat{M}_1-\bb{E}X\right| \le \epsilon_1$ and $\left|\hat{M}_2-{\rm var}X\right| \le \epsilon_2.$
Therefore, we must have 
$
\left| \hat{\mu}_1+\hat{\mu}_2-\mu_1-\mu_2 \right| \le 2\epsilon_1$,
$\left| (\hat{\mu}_1-\hat{\mu}_2)^2- (\mu_1-\mu_2)^2 \right| \le 4\epsilon_2.
$
We can factorize the left hand side of the second equation in the following way:
$
\left| \hat{\mu}_1-\hat{\mu}_2-\mu_1+\mu_2 \right| \left| \hat{\mu}_1-\hat{\mu}_2+\mu_1-\mu_2 \right| \le 4\epsilon_2.
$
Notice that one of the factors must be less than $2\sqrt{\epsilon_2}$. Without loss of generality, let us assume that 
$\left| \hat{\mu}_1-\hat{\mu}_2-\mu_1+\mu_2 \right| \le 2\sqrt{\epsilon_2}.$
This, along with the fact $\left| \hat{\mu}_1+\hat{\mu}_2-\mu_1-\mu_2 \right| \le 2\epsilon_1$ implies that (by adding and subtracting)
$
\left|\hat{\mu}_i-\mu_i\right| \le 2\epsilon_1+2\sqrt{\epsilon_2} \quad \forall i=1,2.
$
\end{proof}
\begin{algorithm}[htbp]
\caption{ \textsc{Method of moments($\f{x},\sigma,T,B$)} Estimate the means $\langle \f{x},\f{\beta}^1 \rangle$, $\langle \f{x},\f{\beta}^2 \rangle$ for a query $\f{x}$ \label{algo:mom}}
\begin{algorithmic}[1]
\REQUIRE An oracle $\ca{O}$ which when queried with a vector $\f{x} \in \bb{R}^n$ returns $\langle \f{x},\f{\beta} \rangle +\ca{N}(0,\sigma^2)$ where $\f{\beta}$ is sampled uniformly from $\{\f{\beta}^1,\f{\beta}^2\}$.
\FOR{$i=1,2,\dots,T$}
\STATE Query the oracle $\ca{O}$ with $\f{x}$ and obtain a response $y^i$.
\ENDFOR
\STATE Compute $\hat{M}_1,\hat{M}_2$ (estimates of $\bb{E}_{X\sim \ca{M}}X,{\rm var}_{X\sim \ca{M}}X$ respectively) using Algorithm \textsc{Estimate}($\f{x},T,B$).
\STATE Solve for $\hat{\mu}_1,\hat{\mu}_2$ in the system of equations $\hat{\mu}_1+\hat{\mu}_2=2\hat{M}_1, \; (\hat{\mu}_1-\hat{\mu}_2)^2=4\hat{M}_2-4\sigma^2$.
\STATE Return $\hat{\mu}_1,\hat{\mu}_2$.
\end{algorithmic}
\end{algorithm}
\subsubsection{Fitting a single Gaussian}
In the situation when both the variance $\sigma^2$ of each component in $\ca{M}$ and the separation between the means $|\mu_1-\mu_2|$ are very small, fitting a single  Gaussian $\ca{N}(\hat{\mu},\sigma^2)$ to the samples obtained from $\ca{M}$ works better than the aforementioned techniques. The procedure to compute $\hat{M}_1$, an estimate of $\bb{E}_{X \sim \ca{M}} X=(\mu_1+\mu_2)/2$ is adapted from \cite{daskalakis2016ten} and is described in Algorithm \ref{algo:single}. Notice that Algorithm \ref{algo:single} is different from the naive procedure (averaging all samples) described in Algorithm \ref{algo:estimate} for estimating the mean of the mixture. The sample complexity for the naive procedure (see Lemma \ref{lem:mean}) scales with the gap $|\mu_1-\mu_2|$ even when the variance $\sigma^2$ is small which is undesirable. In stead we have the following lemma.
\begin{lemma}[Lemma 5 in \cite{daskalakis2016ten}]{\label{thm:mean2}}
With Algorithm \ref{algo:single}, $O\Big(\Big\lceil\sigma^2 \log \eta^{-1}/\epsilon^2 \Big\rceil \Big)$ samples are sufficient to compute $\hat{M}_1$  such that 
$\left|\hat{M}_1-(\mu_1+\mu_2)/2 \right| \le \epsilon$ with probability at least $1-\eta$. 
\end{lemma}
In this case, we will return $\hat{M}_1$ to be estimates of both the means $\mu_1,\mu_2$. 
\begin{algorithm}[htbp]
\caption{ \textsc{Fit a single gaussian($\f{x},T$)} Estimate the means $\langle \f{x},\f{\beta}^1 \rangle$, $\langle \f{x},\f{\beta}^2 \rangle$ for a query $\f{x}$ \label{algo:single}}
\begin{algorithmic}[1]
\REQUIRE An oracle $\ca{O}$ which when queried with a vector $\f{x} \in \bb{R}^n$ returns $\langle \f{x},\f{\beta} \rangle +\ca{N}(0,\sigma^2)$ where $\f{\beta}$ is sampled uniformly from $\{\f{\beta}^1,\f{\beta}^2\}$.
\FOR{$i=1,2,\dots,T$}
\STATE Query the oracle $\ca{O}$ with $\f{x}$ and obtain a response $y^i$.
\ENDFOR
\STATE Set $\hat{Q}_1$ and $\hat{Q}_3$ to be the first and third quartiles of the samples $y^1,y^2,\dots,y^t$ respectively.
\STATE Return $(\hat{Q}_1+\hat{Q}_3)/2$.
\end{algorithmic}
\end{algorithm}

\subsubsection{Choosing appropriate methods}{\label{sec:summ}}
Among the above three methods to learn mixtures, the appropriate algorithm to apply for each parameter regime is listed below.

{\bf Case 1} ($|\mu_1-\mu_2|=\Omega(\sigma)$): We  use the EM algorithm for this regime to recover $\mu_1,\mu_2$. Notice that in this regime, 
by using Theorem \ref{thm:EM} with $\epsilon=\gamma$, we obtain that $O
\Big(\Big\lceil(\sigma^2/\gamma^2)\log 1/\eta \Big\rceil \Big)$ samples are sufficient to recover $\mu_1,\mu_2$ up to an additive error of $\gamma$ with probability at least $1-\eta$.
 
{\bf Case 2} ($\sigma \ge \gamma, |\mu_1-\mu_2|=O(\sigma)$): We  use the method of moments to recover $\mu_1,\mu_2$. In this regime, we must have $\sigma_{\ca{M}}^2=O(\sigma^2)$. Therefore, by using Theorem \ref{thm:mom} with $\epsilon_1=\gamma/4,\epsilon_2=\gamma^2/16$, it is evident that $O\Big(\Big\lceil(\sigma/\gamma)^4 \Big\rceil\log 1/\eta
\Big)$ samples are sufficient to recover $\mu_1,\mu_2$ upto an additive error of $\gamma$ with probability at least $1-\eta$.

{\bf Case 3} ($\sigma \le \gamma, |\mu_1-\mu_2| \le \gamma$): In this setting, we fit a single Gaussian. Using Theorem \ref{thm:mean2} with $\epsilon=\gamma/2$, we will be able to estimate $(\mu_1+\mu_2)/2$ up to an additive error of $\gamma/2$ using $O\Big(\Big\lceil(\sigma^2/\gamma^2)\log 1/\eta \Big\rceil\Big)$ samples. This, in turn implies 
\begin{align*}
|\mu_i-\hat{M}_1| \le \frac{|\mu_1-\mu_2|}{2}+\left|\frac{\mu_1+\mu_2}{2}-\hat{M}_1\right| \le \gamma.
\end{align*} 
for $i \in \{1,2\}$ and therefore both the means $\mu_1,\mu_2$ are recovered up to an additive error of $\gamma$. Note that these three cases covers all possibilities.

\subsubsection{Test for appropriate method}
Now, we describe a test to infer which parameter regime we are in and therefore which algorithm to use. The final algorithm to recover the means $\mu_1,\mu_2$ from $\ca{M}$ including the test is described in Algorithm \ref{algo:test}. We have the following result, the proof of which is delegated to appendix \ref{sec:test_proof}.
\begin{algorithm}[t]
\caption{ \textsc{Test and estimate($\f{x},\sigma,\gamma,\eta$)} Test for the correct parameter regime and apply the parameter estimation algorithm accordingly for a query $\f{x}$ \label{algo:test}}
\begin{algorithmic}[1]
\REQUIRE An oracle $\ca{O}$ which when queried with a vector $\f{x} \in \bb{R}^n$ returns $\langle \f{x},\f{\beta} \rangle +\ca{N}(0,\sigma^2)$ where $\f{\beta}$ is sampled uniformly from $\{\f{\beta}^1,\f{\beta}^2\}$.
\STATE Set $T=O\Big(\Big\lceil \log \eta^{-1} \Big\rceil \Big)$.
\FOR{$i=1,2,\dots,T$}
       \STATE Query the oracle $\ca{O}$ with $\f{x}$ and obtain a response $y^i$.
\ENDFOR
\STATE Compute $\tilde{\mu}_1,\tilde{\mu}_2$ by running Algorithm \textsc{Method of moments} ($\f{x},\sigma,T,72\log n$).
\IF{$\sigma > \gamma $ and $|\tilde{\mu}_1-\tilde{\mu}_2| \le 15\sigma/32$}
       \STATE Compute $\hat{\mu}_1,\hat{\mu}_2$ by running Algorithm \textsc{Method of moments} ($\f{x},\sigma,O\Big(\Big\lceil(\sigma/\gamma)^4 \Big\rceil\log 1/\eta,72\log n\Big)$.
\ELSIF{$\sigma \le \gamma$ and $|\tilde{\mu}_1-\tilde{\mu}_2| \le 15\gamma/32$}       
       \STATE Compute $\hat{\mu}_1,\hat{\mu}_2$ by running Algorithm \textsc{Fit a single gaussian} ($\f{x},O\Big(\Big\lceil (\sigma^2/\gamma^2)\log 1/\eta \Big\rceil \Big)$.
\ELSE 
       \STATE Compute $\hat{\mu}_1,\hat{\mu}_2$ by running Algorithm \textsc{EM}($\f{x},\sigma,O\Big(\Big\lceil(\sigma^2/\gamma^2)\log 1/\eta \Big\rceil \Big)$.
\ENDIF
\STATE Return $\hat{\mu}_1,\hat{\mu}_2$.       
\end{algorithmic}
\end{algorithm}
\begin{lemma}{\label{lem:test}}
The number of samples required for Algorithm \ref{algo:test}  to infer the parameter regime correctly with probability at least $1-\eta$ is atmost $O(\log \eta^{-1})$. 
\end{lemma}
\subsection{Alignment}\label{sec:al}
For a query $\f{x}^i,i \in [m]$, let us introduce the following notations for brevity:
\begin{align*}
\mu_{i,1} := \langle \f{x}^i,\beta_1 \rangle \quad \mu_{i,2} := \langle \f{x}^i,\beta_2 \rangle.
\end{align*}
Now, using Algorithm \ref{algo:test}, we can compute $(\hat{\mu}_{i,1},\hat{\mu}_{i,2})$ (estimates of $\mu_{i,1},\mu_{i,2}$) using a batchsize of $T_i$ such that
$
\left|\hat{\mu}_{i,j}-\mu_{i,\pi_i(j)} \right| \le \gamma \quad \forall \; i \in [m],j \in \{1,2\},
$    
where $\pi_i:\{1,2\}\rightarrow\{1,2\}$ is a permutation on $\{1,2\}$.

\begin{algorithm}[htbp]
\caption{ \textsc{Align pair($\f{x}^i,\f{x}^j,\{\hat{\mu}_{s,t}\}_{s=i,j \; t=1,2},\sigma,\gamma,\eta$)} Align the mean estimates for $\f{x}^i$ and $\f{x}^j$. \label{algo:aligntwo}}
\begin{algorithmic}[1]
\STATE Recover $\hat{\mu}_{\s{sum},1},\hat{\mu}_{\s{sum},2}$ using Algorithm \textsc{Test and estimate} ($\f{x}^i+\f{x}^j,\sigma,\gamma,\eta)$.
\STATE Recover $\hat{\mu}_{\s{diff},1},\hat{\mu}_{\s{diff},2}$ using Algorithm \textsc{Test and estimate} ($\f{x}^i-\f{x}^j,\sigma,\gamma,\eta)$.
\IF{$|\hat{\mu}_{\s{sum},1}-\hat{\mu}_{i,p}-\hat{\mu}_{j,q}|\le 3\gamma$ such that $p,q \in \{1,2\}$ is unique}
      \STATE {\bf if} $p==q$ {\bf then} Return TRUE {\bf else} Return FALSE {\bf end if}
\ELSE
     \STATE Find $p,q$ such that $|\hat{\mu}_{\s{diff},1}-\hat{\mu}_{i,p}+\hat{\mu}_{j,q}|\le 3\gamma$ for $p,q \in \{1,2\}$.
     \STATE {\bf if} $p==q$ {\bf then} Return TRUE {\bf else} Return FALSE {\bf end if}
\ENDIF                            
\end{algorithmic}
\end{algorithm}
The most important step in our process is to separate the estimates of the means according to the generative unknown sparse vectors $(\f{\beta}^1$ and $\f{\beta}^2)$ (i.e., alignment). Formally,  we construct two $m$-dimensional vectors $\f{u}$ and $\f{v}$ such that, for all $i \in [m]$ the following hold:
\begin{itemize}[leftmargin=*,noitemsep,topsep=0em]
\item The $i^{\s{th}}$ elements of $\f{u}$ and  $\f{v}$, i.e., $u_i$ and $v_i$, are   $\hat{\mu}_{i,1}$ and $\hat{\mu}_{i,2}$ (but may not be respectively).
\item Moreover, we must have the $u_i$ and $v_i$ to be good estimates of $\langle \f{x}^i, \f{\beta}^{\pi(1)} \rangle$ and $\langle \f{x}^i, \f{\beta}^{\pi(2)} \rangle$ respectively i.e.
$
\left|u_i-\langle \f{x}^i, \f{\beta}^{\pi(1)} \rangle \right|\le  10\gamma$ ;
$\left|v_i-\langle \f{x}^i, \f{\beta}^{\pi(2)} \rangle \right| \le 10\gamma  
$
for all $i \in [m]$ where $\pi:\{1,2\}\rightarrow \{1,2\}$ is some permutation of $\{1,2\}$.
\end{itemize} 
In essence, for the alignment step, we want to find out all permutations $\pi_i, i \in [m]$.
First, note  that the aforementioned objective for the $i^{\s{th}}$ element of $\f{u},\f{v}$ is trivial 
when $|\mu_{i,1}-\mu_{i,2}| \le 9\gamma$. To see this, suppose $\pi_i$ is the identity permutation without loss of generality. In that case, we have for $\hat{\mu}_{i,1}$, $|\hat{\mu}_{i,1}-\mu_{i,1}| \le \gamma$ and
$|\hat{\mu}_{i,1}-\mu_{i,2}| \le |\hat{\mu}_{i,1}-\mu_{i,1}|+|\mu_{i,1}-\mu_{i,2}| \le 10\gamma.$
Similar guarantees also hold for $\hat{\mu}_{i,2}$ and therefore the choice of the $i^{\s{th}}$ element of $\f{u},\f{v}$ is trivial. This conclusion implies that the interesting case is only for those queries $\f{x}^i$ when $|\mu_{i,1}-\mu_{i,2}|\ge 9 \gamma$. In other words, this objective is equivalent to separate out the permutations $\{\pi_i\}_{i=1}^{m}$ for $i: |\mu_{i,1}-\mu_{i,2}|\ge 9 \gamma$ into two groups such that all the permutations in each group are the same. 


\subsubsection{Alignment for two queries}
Consider two queries $\f{x}^1,\f{x}^2$ such that $\left|\mu_{i,1}-\mu_{i,2}\right| \ge 9\gamma$ for $i=1,2$.  In this section, we will show how we can infer if $\pi_1$ is same as $\pi_2$. Our strategy is to make two additional batches of queries corresponding to $\f{x}^1+\f{x}^2$ and $\f{x}^1-\f{x}^2$ (of size $T_{1,2}^{\s{sum}}$ and $T_{1,2}^{\s{diff}}$ respectively) which we shall call the \textit{sum} and \textit{difference} queries. Again, let us introduce the following notations:
$
\mu_{\s{sum},1} = \langle \f{x}^1+\f{x}^2,\beta_1 \rangle \quad \mu_{\s{sum},2} = \langle \f{x}^1+\f{x}^2,\beta_2 \rangle$,
$
\mu_{\s{diff},1} = \langle \f{x}^1-\f{x}^2,\beta_1 \rangle \quad \mu_{\s{diff},2} = \langle \f{x}^1-\f{x}^2,\beta_2 \rangle.
$ 
As before, using Algorithm \ref{algo:test}, we can compute $(\hat{\mu}_{\s{sum},1},\hat{\mu}_{\s{sum},2})$ (estimates of $\mu_{\s{sum},1},\mu_{\s{sum},2}$) and $(\hat{\mu}_{\s{diff},1},\hat{\mu}_{\s{diff},2})$ (estimates of $\mu_{\s{diff},1},\mu_{\s{diff},2}$) using a batchsize of $T_{1,2}^{\s{sum}}$ and $T_{1,2}^{\s{diff}}$ for the \textit{sum} and \textit{difference} query respectively such that 
$
\left|\hat{\mu}_{\s{sum},j}-\mu_{\s{sum},\pi_{\s{sum}}(j)} \right| \le \gamma \quad \text{for} \; j \in \{1,2\}$ and
$\left|\hat{\mu}_{\s{diff},j}-\mu_{\s{diff},\pi_{\s{diff}}(j)} \right| \le \gamma \quad \text{for} \; j \in \{1,2\} 
$   
where $\pi_{\s{sum}},\pi_{\s{diff}}:\{1,2\}\rightarrow\{1,2\}$ are again unknown permutations of $\{1,2\}$. We show the following lemma.
\begin{lemma}{\label{lem:align}}
We can infer, using Algorithm \ref{algo:aligntwo}, if $\pi_1$ and $\pi_2$ are same using the estimates $\hat{\mu}_{\s{sum},i},\hat{\mu}_{\s{diff},i}$ provided $\left|\mu_{i,1}-\mu_{i,2}\right| \ge 9\gamma$, $i=1,2$.
\end{lemma}

The proof of this lemma is delegated to appendix \ref{sec:align_proof} and we provide an outline over here. In Algorithm \ref{algo:aligntwo}, we first choose one value from $\{\hat{\mu}_{\s{sum},1}, \hat{\mu}_{\s{sum},2}\}$  (say $z$) and we check if we can choose one element (say $a$) from the set $\{\hat{\mu}_{1,1}, \hat{\mu}_{1,2}\}$ and one element $\{\hat{\mu}_{2,1},\hat{\mu}_{2,2}\}$ (say $b$) in exactly one way such that $|z-a-b| \le 3\gamma$. If that is true, then we infer that the tuple $\{a,b\}$ are estimates of the same unknown vector and accordingly return if $\pi_1$ is same as $\pi_2$. If not possible, then we choose one value from $\{\hat{\mu}_{\s{diff},1}, \hat{\mu}_{\s{diff},2}\}$  (say $z'$) and again we check if we can choose one element (say $c$) from the set $\{\hat{\mu}_{1,1}, \hat{\mu}_{1,2}\}$ and one element from $\{\hat{\mu}_{2,1},\hat{\mu}_{2,1}\}$ (say $d$) in exactly one way such that $|z'-c-d| \le 3\gamma$. If that is true,
then we infer that $\{c,d\}$ are estimates of the same unknown vector and accordingly return if $\pi_1$ is same as $\pi_2$.
It can be shown that we will succeed in this step using at least one of the sum or difference queries.

\subsubsection{Alignment for all queries}{\label{sec:align}}
We will align the mean estimates for all the queries $\f{x}^1,\f{x}^2,\dots,\f{x}^m$ by aligning one pair at a time. This routine is summarized in Algorithm~\ref{algo:alignall}, which works when $\gamma \le \frac{13\sqrt{2}}{\sqrt{\pi}}\|\beta^1-\beta^2\|_2\approx 0.096\|\beta^1-\beta^2\|_2$. To understand the routine, we start with the following technical lemma:
\begin{lemma}{\label{lem:exist}}
Let, $\gamma \le \frac{13\sqrt{2}}{\sqrt{\pi}}\|\beta^1-\beta^2\|_2$. For $m'=\Big\lceil\log \eta^{-1}/\log \frac{\sqrt{\pi}||\f{\beta}^1-\f{\beta}^2||_2}{13\sqrt{2}\gamma}\Big\rceil$,
there exists a query $\f{x}^{i^{\star}}$ among $\{\f{x}^i\}_{i=1}^{m'}$ such that $|\mu_{i^{\star},1}-\mu_{i^{\star},2}|\ge 13 \gamma$   with probability at least $1-\eta$.
\end{lemma}

\begin{algorithm}[htbp]
\caption{ \textsc{Align all($\{\f{x}^i\}_{i\in [m]},\{\hat{\mu}_{s,t}\}_{s\in [m] \; t=1,2},\sigma,\gamma,\eta$)} Align mean estimates for all queries $\{\f{x}^i\}_{i=1}^{m}$. \label{algo:alignall}}
\begin{algorithmic}[1]
\STATE \textbf{Initialize:} $\f{u}, \f{v}$ to be $m$-dimensional all zero vector.
\STATE Set $m'= \Big\lceil \log \eta^{-1}/\log \frac{\sqrt{\pi}||\f{\beta}^1-\f{\beta}^2||_2}{13\sqrt{2}\gamma}\Big\rceil$
\FOR{$i=1,2,\dots,m$}
\FOR{$j=1,2,\dots,m',j \neq i$}
      \STATE Run Algorithm \textsc{Align pair} ($\f{x}^i,\f{x}^j,$ $\{\hat{\mu}_{s,t}\}_{\substack{s=i,j \\ t=1,2}},\sigma,\gamma,\eta)$ and store the output.
\ENDFOR
\ENDFOR      
\STATE Identify $\f{x}^{p}$ from $p \in [m']$ such that $|\hat{\mu}_{p,1}-\hat{\mu}_{p,2}| \ge 11 \gamma$.
\STATE Set $u_p:=\f{u}[p]=\hat{\mu}_{p,1}$ and $v_p:=\f{v}[p]=\hat{\mu}_{p,2}$
\FOR{$i=1,2,\dots,m, i\neq p$}
        \IF{Output of Algorithm \ref{algo:aligntwo} for $\f{x}^i$ and $\f{x}^p$ is TRUE}
              \STATE Set $\f{u}[i]=\hat{\mu}_{i,1}$ and $\f{v}[i]=\hat{\mu}_{i,2}$.
        \ELSE
              \STATE Set $\f{u}[i]=\hat{\mu}_{i,2}$ and $\f{v}[i]=\hat{\mu}_{i,1}$.
        \ENDIF
\ENDFOR
\STATE Return $\f{u},\f{v}$.            
\end{algorithmic}
\end{algorithm}
The proof of this lemma is delegated to Appendix \ref{sec:exist}. Now, for $i \in [m'], j \in [m]$ such that $i \neq j$, we will align $\f{x}^i$ and $\f{x}^j$ using Algorithm \ref{algo:aligntwo} and according to Lemma \ref{lem:align}, this alignment procedure will succeed for all such pairs where $|\mu_{i,1}-\mu_{i,2}|,|\mu_{j,1}-\mu_{j,2}| \ge 9\gamma$ with probability at least $1-mm'\eta$ (using a union bound). Note that according to Lemma \ref{lem:exist}, there must exist a query $\f{x}^{i^{\star}} \in \{\f{x}^i\}_{i=1}^{m'}$ for which $|\mu_{i^{\star},1}-\mu_{i^{\star},2}|\ge 13 \gamma$. This implies that for some $i^{\star} \in [m']$, we must have
$
|\hat{\mu}_{i^{\star},1}-\hat{\mu}_{i^{\star},2}| \ge |\mu_{i^{\star},1}-\mu_{i^{\star},2}|-|\hat{\mu}_{i^{\star},1}-\mu_{i^{\star},1}| -|\hat{\mu}_{i^{\star},2}-\mu_{i^{\star},2}| \ge 11\gamma.
$
Therefore, we can identify at least one query $\f{x}^{\tilde{i}}$ for $\tilde{i}\in [m']$ such that $|\hat{\mu}_{\tilde{i},1}-\hat{\mu}_{\tilde{i},2}| \ge 11 \gamma$. However, this implies that
$
|\mu_{\tilde{i},1}-\mu_{\tilde{i},2}| \ge |\hat{\mu}_{\tilde{i},1}-\hat{\mu}_{\tilde{i},2}|-|\mu_{\tilde{i},1}-\hat{\mu}_{\tilde{i},1}| 
-|\mu_{\tilde{i},2}-\hat{\mu}_{\tilde{i},2}| \ge 9\gamma.
$
Therefore we will be able to infer for every query $\f{x}^i, i\in [m]$ for which $|\mu_{i,1}-\mu_{i,2}|\ge 9\gamma$  if $\pi_i$ is same as $\pi_{\tilde{i}}$. Now, we are ready to put everything together and provide the proof for the main result (Thm. \ref{thm:main}).

\subsection{Proof of Theorem \ref{thm:main} ($\gamma < \left|\left|\beta^1-\beta^2\right|\right|_2/2$)}\label{sec:tm}
The overall recovery procedure is described as Algorithm~\ref{algo:main}. Since this algorithm crucially uses Algorithm~\ref{algo:alignall}, it works only when $\gamma \le 0.096\left|\left|\beta^1-\beta^2\right|\right|_2$; so assume that to hold for now.
We will start by showing that for any two Gaussian queries, the samples are far enough (a simple instance of Gaussian anti-concentration).
\begin{algorithm}[htbp]
\caption{ \textsc{Recover unknown vectors($\sigma,\gamma$)} Recover the unknown vectors $\f{\beta}^1$ and $\f{\beta}^2$ \label{algo:main}}
\begin{algorithmic}[1]
\STATE Set $m=c_sk \log n$.
\STATE Sample $\f{x}^1,\f{x}^2,\dots,\f{x}^m \sim \ca{N}(\f{0},\f{I}_n)$ independently.
\FOR{$i=1,2,\dots,m$}
         \STATE Compute $\hat{\mu}_{i,1},\hat{\mu}_{i,2}$ by running Algorithm \textsc{Test and estimate}($\f{x}^i,\sigma,\gamma,n^{-2}$).
\ENDFOR
\STATE Construct $\f{u},\f{v}$ by running Algorithm \textsc{Align all}($\{\f{x}^i\}_{i\in [m]},\{\hat{\mu}_{s,t}\}_{\substack{s\in [m] \\ t=1,2}},\sigma,\gamma,\eta$).
\STATE Set $\f{A}$ to be the $m \times n$ matrix such that its $i^{\s{th}}$ row is $\f{x}^i$, with each entry normalized by $\sqrt{m}$.
\STATE Set $\hat{\f{\beta}}^1$ to be the solution of the optimization problem $\min_{\f{z} \in \bb{R}^n} ||\f{z}||_1 \; \text{s.t.} \; ||\f{A}\f{z}-\frac{1}{\sqrt{m}}\f{u}||_2 \le 10\gamma$
\STATE Set $\hat{\f{\beta}}^2$ to be the solution of the optimization problem $\min_{\f{z} \in \bb{R}^n} ||\f{z}||_1 \; \text{s.t.} \; ||\f{A}\f{z}-\frac{1}{\sqrt{m}}\f{v}||_2 \le 10\gamma$
\STATE Return $\hat{\f{\beta}}^1,\hat{\f{\beta}}^2$.
\end{algorithmic}
\end{algorithm}

\begin{lemma}{\label{lem:prob}}
For all queries $\f{x}$ designed in Algorithm \ref{algo:main}, for any constant $c_1>0$, and some $c_2$ which is a function of $c_1$,
\begin{align*}
\Pr(\left| \langle \f{x},\f{\beta}^1 \rangle- \langle \f{x},\f{\beta}^2 \rangle \right| \le c_1\sigma)) \le \frac{c_2\sigma}{||\f{\beta}^1-\f{\beta}^2||_2}. 
\end{align*}
\end{lemma}
The proof of this lemma is delegated to Appendix \ref{sec:exist}. 
Now the theorem is proved via a series of claims.

\begin{claim}
The expected batchsize for any query designed in Algorithm \ref{algo:main} is  $O\Big(\Big\lceil \frac{\sigma^5 }{\gamma^4 ||\f{\beta}^1-\f{\beta}^2||_2}+\frac{\sigma^2}{\gamma^2} \Big\rceil \log \eta^{-1}\Big)$.
\end{claim}

\begin{proof}
In Algorithm \ref{algo:main}, we make $m$ batches of queries corresponding to $\{\f{x}^i\}_{i=1}^{m}$ and $mm'$ batches of queries corresponding to $\{\f{x}^i+\f{x}^j\}_{i=1,j=1,i \neq j}^{i=m,j=m'}$ and $\{\f{x}^i-\f{x}^j\}_{i=1,j=1,i \neq j}^{i=m,j=m'}$. Recall that the batchsize corresponding to $\f{x}^i,\f{x}^i+\f{x}^j,\f{x}^i-\f{x}^j$ is denoted by $T_i$,$T_{i,j}^{\s{sum}}$ and $T_{i,j}^{\s{diff}}$ respectively.
Recall from Section \ref{sec:summ},  we will use method of moments or or fit a single Gaussian (Case 2 and 3 in Section \ref{sec:summ}) to estimate the means when the difference between the means is $O(\sigma)$. By Lemma \ref{lem:prob}, this happens with probability $O(\sigma/||\f{\beta}^1-\f{\beta}^2||_2)$.
Otherwise we will use the EM algorithm (Case 1 in Section \ref{sec:summ}). 
or fit a single gaussian, both of which require a batchsize of at most $O\Big(\Big\lceil\sigma^2 /\gamma^2 \Big\rceil\log \eta^{-1}\Big)$.
We can conclude that the expected size of any of the aforementioned batchsize is bounded from above as the following:
$
\bb{E}T \le O\Big(\Big\lceil\frac{\sigma^5 }{\gamma^4 ||\f{\beta}^1-\f{\beta}^2||_2}+\frac{\sigma^2}{\gamma^2} \Big\rceil\log \eta^{-1}\Big)  
$
where $T \in \{T_i\}\cup\{T_{i,j}^{\s{sum}}\}\cup\{T_{i,j}^{\s{diff}}\}$ so that we can recover all the mean estimates upto an an additive error of $\gamma$ with probability at least $1-O(mm'\eta)$. 
\end{proof}

\begin{claim}
Algorithm \ref{algo:alignall} returns two vectors $\f{u}$ and $\f{v}$ of length $m$ each 
such that 
\begin{align*}
\left|\f{u}[i]-\langle \f{x}^i, \f{\beta}^{\pi(1)} \rangle \right|\le  10\gamma;\left|\f{v}[i]-\langle \f{x}^i, \f{\beta}^{\pi(2)} \rangle \right| \le 10\gamma  
\end{align*}
for some permutation $\pi:\{1,2\}\rightarrow \{1,2\}$ for all $i \in [m]$ with probability at least $1-\eta$.
\end{claim}
The proof of this claim directly follows from the discussion in Section \ref{sec:align}.

The matrix $\f{A}$ is size $m \times n$ whose $i^{\s{th}}$ row is the query vector $\f{x}^i$ normalized by $\sqrt{m}$. .
\begin{claim}
We must have 
\begin{align*}
||\f{A}\f{\beta}^{\pi(1)}-\frac{\f{u}}{\sqrt{m}}||_2 \le 10\gamma \; \& \; ||\f{A}\f{\beta}^{\pi(2)}-\frac{\f{v}}{\sqrt{m}}||_2 \le 10\gamma.
\end{align*}
\end{claim}
\begin{proof}
The proof of this claim follows from the fact that after normalization by $\sqrt{m}$, the error in each entry is also normalized by $\sqrt{m}$ and is therefore at most $10\gamma/\sqrt{m}$. Hence the $\ell_2$ difference is at most $10\gamma$.
\end{proof}
It is known that for  $m \ge c_s k \log n$ where  $c_s>0$ is some appropriate constant, the matrix $\f{A}$ satisfy {\em restricted isometric property} of order $2k,$ which means for any exactly $2k$-sparse vector $\f{x}$ and a constant $\delta,$ we have $|\|\f{A}\f{x}\|_2^2 - \|\f{x}\|_2^2| \le \delta \|\f{x}\|_2^2$ ~cf.\cite{baraniuk2008simple}.

 We now solve the following convex optimization problems, standard recovery method called basis pursuit:
\begin{align*}
&\hat{\f{\beta}}^{\pi(1)}= \min_{\f{z} \in \bb{R}^n} ||\f{z}||_1 \; \text{s.t.} \; ||\f{A}\f{z}-\frac{\f{u}}{\sqrt{m}}||_2 \le 10\gamma \\
&\hat{\f{\beta}}^{\pi(2)}= \min_{\f{z} \in \bb{R}^n} ||\f{z}||_1 \; \text{s.t.} \; ||\f{A}\f{z}-\frac{\f{v}}{\sqrt{m}}||_2 \le 10\gamma 
\end{align*}
to recover $\hat{\f{\beta}}^{\pi(1)},\hat{\f{\beta}}^{\pi(2)}$, estimates of $\f{\beta}^1,\f{\beta}^2$ having the guarantees given in Theorem \ref{thm:main} (see, Thm. 1.6 in \cite{boche2015survey}). The expected query complexity is
$
O\Big(mm'\log \eta^{-1} \Big\lceil\frac{\sigma^5 }{\gamma^4 ||\f{\beta}^1-\f{\beta}^2||_2}+\frac{\sigma^2}{\gamma^2}\Big\rceil\Big)$
Substituting 
$
m=O(k \log n)$,  $m'=O\Big(\Big\lceil\frac{\log \eta^{-1}}{\log \frac{||\f{\beta}^1-\f{\beta}^2||_2}{\gamma}} \Big\rceil \Big)$ and $\eta = (mm'\log n)^{-1},$
we obtain the total query complexity  
\begin{align*}
&O\Bigg(k\log n\log k \Big\lceil\frac{\log k}{\log (\|\f{\beta}^1-\f{\beta}^2\|_2/\gamma)}\Big\rceil \\
&\times \Big\lceil\frac{\sigma^5}{\gamma^4||\f{\beta}^1-\f{\beta}^2||_2}+\frac{\sigma^2}{\gamma^2}\Big\rceil\Bigg)
\end{align*}
and the error probability to be $o(1)$. We can just substitute the definition of $\mathsf{NF}$ and notice that $\mathsf{SNR} = \|\f{\beta}^1-\f{\beta}^2\|_2^2/ \sigma^2$ to obtain the query complexity promised in Theorem~\ref{thm:main}. Note that, we have assumed $k = \Omega(\log n)$ above.

It remains to be proved that the same (orderwise) number of samples is sufficient to recover both unknown vectors
 with high probability. For each query $\f{x}$ designed in Algorithm \ref{algo:main}, consider the indicator random variable $Y_i=\mathds{1}[|\mu_{i,1}-\mu_{i,2}|=\Omega(\sigma)]$. The total number of queries for which this event is true (given by $\sum_i Y_i$) is sampled according to the binomial distribution $\s{Bin}(mm',O(\sigma/||\f{\beta}^1-\f{\beta}^2||_2))$ and therefore concentrates tightly around its mean. A simple use of Chernoff bound leads to the desired result. 
 
While we have proved the theorem for any $\gamma \le 0.096\left|\left|\beta^1-\beta^2\right|\right|_2$, it indeed holds for any   $\gamma =c'\left|\left|\beta^1-\beta^2\right|\right|_2,$ where $c'$ is a constant strictly less than $1.$  If the desired $\gamma > 0.096\left|\left|\beta^1-\beta^2\right|\right|_2$, then one can just define $\gamma'= 0.096\left|\left|\beta^1-\beta^2\right|\right|_2$ and obtain a precision $\gamma'$ which is a constant factor within $\gamma$. Since the quantity $\mathsf{NF}$ defined with $\gamma'$ is also within a constant factor of  the original $\mathsf{NF}$, the sample complexity can also change by at most a constant factor.
 
 \subsection{Proof of Theorem \ref{thm:main} ($\gamma = \Omega\left( \left|\left|\beta^1-\beta^2\right|\right|_2\right)$)}

The proof of Theorem \ref{thm:main} for the case when recovery precision $\gamma = \Omega(\left|\left|\beta^1-\beta^2\right|\right|_2)$ follows by fitting a single Gaussian through all the samples. The algorithm for this case, and the proof, are delegated to Appendix~\ref{sec:thm1_rem}. 

\section{Conclusion}
In this paper we have improved the recent results by \cite{yin2019learning} and \cite{KrishnamurthyM019} for learning a mixture of sparse linear regressions when features can be designed and queried with for the labels. While our results are rigorously proved for two unknown sparse models, we believe extending to more than two models will be possible, and the key components are already present in our paper. Whether it will be an exercise in technicality or some key insights can be gained is unclear.

While our paper is theoretical, an important future work will be to find interesting use cases. A potential application of the query-based setting is to recommendation systems, where the goal is to identify the factors governing the preferences of individual members of a group via crowdsourcing while also preserving the anonymity of their responses. We are currently pursuing this line of applications.

\paragraph{Acknowledgements:} This work is supported in parts by NSF awards 1642658, 1909046 and 1934846.

\bibliographystyle{plain}

\newpage
\onecolumn
\appendix
\section{Proofs of Lemma \ref{lem:mean} and \ref{lem:var}}{\label{sec:proof1}}

Let $X$ be a random variable which is distributed according to $\ca{M}$ and suppose we obtain $T$ samples $y_1,y_2,\dots,y_T \sim \ca{M}$. We will divide these $T$ samples into $B:=\Big\lceil T/t \Big\rceil$ batches each of size $t$. In that case let us denote $S^j_{1,t}$ and $S^j_{2,t}$ to be the sample mean and the sample variance of the $j^{th}$ batch i.e. 
\begin{align*}
S^j_{1,t} = \sum_{i \in \text{Batch } j} \frac{y_i}{t} \quad \quad \text{and} \quad \quad S^j_{2,t} =\frac{1}{t-1} \sum_{i \in \text{Batch } j} (y_i-(S^j_{1,t}))^{2}  .
\end{align*}
We will estimate the true mean $\bb{E}X$ and the true variance ${\rm var}X$ by computing $\hat{M}_1$ and $\hat{M}_2$ respectively (See Algorithm \ref{algo:estimate}) where
\begin{align*}
\hat{M}_1\triangleq \s{median}(\{S^j_{1,t}\}_{j=1}^{B}) \quad \text{and} \quad \hat{M}_2 \triangleq \s{median}(\{S^j_{2,t}\}_{j=1}^{B}).
\end{align*}

\begin{proof}[Proof of Lemma \ref{lem:mean}]
For a fixed batch $j$, we can use Chebychev's inequality to say that 
\begin{align*}
\Pr\Big(\left|S^j_{1,t}-\bb{E}X\right| \ge \epsilon_1 \Big) \le \frac{{\rm var}X}{t\epsilon_1^2} 
\end{align*} 
We have 
\begin{align*}
{\rm var}X = \bb{E}X^2 - (\bb{E}X)^{2} = \frac{1}{2}\Big(2\sigma^2+\mu_1^2+\mu_2^2 \Big)-\frac{1}{4}(\mu_1+\mu_2)^{2} = \sigma^2+\frac{(\mu_1-\mu_2)^2}{4}
\end{align*}
Noting that we must have $t \ge 1$ as well, we obtain
\begin{align*}
\Pr\Big(\left|S^j_{1,t}-\bb{E}X\right| \ge \epsilon_1 \Big) \le \frac{\sigma^2+(\mu_1-\mu_2)^2/4}{t\epsilon_1^2} \le \frac{1}{3} 
\end{align*} 
for $t=O(\Big\lceil(\sigma^2+(\mu_1-\mu_2)^2)/\epsilon_1^2 \Big\rceil)$. Therefore for each batch $j$, we define an indicator random variable $Z_j=\mathds{1}[\left|S^j_{1,t}-\bb{E}X\right| \ge \epsilon_1]$ and from our previous analysis we know that the probability of $Z_j$ being \textsc{1} is less than $1/3$. It is clear that $\bb{E} \sum_{j=1}^{B} Z_j \le B/3$ and on the other hand $|\hat{M}_1-\bb{E}X| \ge \epsilon_1$ iff  $\sum_{j=1}^{B} Z_j \ge B/2$. Therefore, using the Chernoff bound, we have
\begin{align*}
\Pr\Big(\left|\hat{M}_1-\bb{E}X \right| \ge \epsilon_1\Big)\le Pr \Big(\left|\sum_{j=1}^{B} Z_j-\bb{E}\sum_{j=1}^{B} Z_j \right| \ge \frac{\bb{E}\sum_{j=1}^{B} Z_j}{2} \Big) \le 2e^{-B/36}. 
\end{align*}
Hence, for $B=36\log \eta^{-1}$, the estimate $\hat{M}_1$ is atmost $\epsilon_1$ away from the true mean $\hat{M}_1$ with probability at least $1-2\eta$. Therefore the total sample complexity required is $  T=O(\log \eta^{-1}\Big\lceil(\sigma^2+(\mu_1-\mu_2)^2)/\epsilon_1^2 \Big\rceil)$ proving the lemma.
\end{proof}

\begin{proof}[Proof of Lemma \ref{lem:var}]
We have
\begin{align*}
\bb{E}S^j_{2,t} &=  \bb{E}\frac{1}{t-1}\sum_{i \in \text{Batch } j}(y_i-(S^j_{1,t}))^{2} \\
&= \bb{E} \frac{1}{t(t-1)}\sum_{\substack{i_1,i_2 \in \text{Batch } j \\ i_1<i_2 } }(y_{i_1}-y_{i_2})^{2} \\
&= \frac{1}{t(t-1)}\sum_{\substack{i_1,i_2 \in \text{Batch } j \\ i_1<i_2 } } \bb{E} (y_{i_1}-y_{i_2})^{2} \\
&= \frac{1}{t(t-1)}\sum_{\substack{i_1,i_2 \in \text{Batch } j \\ i_1<i_2 } } \bb{E} y^{2}_{i_1}+\bb{E} y^{2}_{i_2}-2\bb{E}\left[y_{i_1}y_{i_2}\right]  \\
&= \frac{1}{t(t-1)}\sum_{\substack{i_1,i_2 \in \text{Batch } j \\ i_1<i_2 } } 2\sigma^2+\mu_1^2+\mu_2^2-\frac{(\mu_1+\mu_2)^2}{2} \\
&= \sigma^2+\frac{(\mu_1-\mu_2)^2}{4} = {\rm var}X.
\end{align*}
Hence the estimator $S^j_{2,t}$ is an unbiased estimator since it's expected value is the true variance of $X$. Again, we must have
\begin{align*}
\bb{E} (S^j_{2,t})^{2}= \bb{E} \frac{1}{t^2(t-1)^2} \Big(\sum_{\substack{i_1,i_2 \in \text{Batch } j \\ i_1<i_2 } }(y_{i_1}-y_{i_2})^{2}\Big)^{2}.
\end{align*}
\begin{claim}{\label{claim:tech}}
We have
\begin{align*}
\bb{E}\left[(y_{i_1}-y_{i_2})^{2}(y_{i_3}-y_{i_4})^{2}\right] \le 48\Big(\sigma^2+\frac{(\mu_1-\mu_2)^2}{4}\Big)^{2}
\end{align*}
for any $i_1,i_2,i_3,i_4$ such that $i_1<i_2$ and $i_3<i_4$.
\end{claim}  
\begin{proof}
In order to prove this claim consider three cases:
\paragraph{Case 1 ($i_1,i_2,i_3,i_4$ are distinct):} In this case, we have that $y_{i_1}-y_{i_2}$ and $y_{i_3}-y_{i_4}$ are independent and therefore,
\begin{align*}
\bb{E}\left[(y_{i_1}-y_{i_2})^{2}(y_{i_3}-y_{i_4})^{2}\right] = \bb{E}\left[(y_{i_1}-y_{i_2})^{2} \right]\bb{E}\left[(y_{i_3}-y_{i_4})^{2} \right] \le 4\Big(\sigma^2+\frac{(\mu_1-\mu_2)^2}{4}\Big)^{2}.
\end{align*}
\paragraph{Case 2 ($i_1=i_3,i_2=i_4$):} In this case, we have
\begin{align*}
\bb{E}\left[(y_{i_1}-y_{i_2})^{2}(y_{i_3}-y_{i_4})^{2}\right]=\bb{E}\left[(y_{i_1}-y_{i_2})^{4}\right].
\end{align*}
Notice that $$y_{i_1}-y_{i_2} \sim \frac{1}{2}\ca{N}(0,2\sigma^2)+\frac{1}{4} \ca{N}(\mu_1-\mu_2,2\sigma^2)+\frac{1}{4} \ca{N}(\mu_2-\mu_1,2\sigma^2)$$
and therefore we get 
\begin{align*}
\bb{E}\left[(y_{i_1}-y_{i_2})^{4}\right]=48\sigma^4+12\sigma^2(\mu_1-\mu_2)^2+\frac{(\mu_1-\mu_2)^4}{2} \le 48\Big(\sigma^2+\frac{(\mu_1-\mu_2)^2}{4}\Big)^{2}.
\end{align*}
\paragraph{Case 3 ($\{i_1,i_2,i_3,i_4\}$ has $3$ unique elements):} Without loss of generality let us assume that $i_1=i_3$. In that case we have 
\begin{align*}
\bb{E}\left[(y_{i_1}-y_{i_2})^{2}(y_{i_1}-y_{i_4})^{2}\right]= \bb{E}_{y_{i_1}}\bb{E}\left[(y_{i_2}-y_{i_1})^{2}(y_{i_4}-y_{i_1})^{2} \mid y_{i_1}\right]
\end{align*}
Notice that for a fixed value of $y_{i_1}$, we must have $y_{i_2}-y_{i_1},y_{i_4}-y_{i_1}$ to be independent and identically distributed i.e.
\begin{align*}
y_{i_2}-y_{i_1},y_{i_4}-y_{i_1} \sim \frac{1}{2}\ca{N}(\mu_1-y_{i_1},\sigma^2)+\frac{1}{2}\ca{N}(\mu_2-y_{i_1},\sigma^2).
\end{align*} 
Therefore,
\begin{align*}
\bb{E}\left[(y_{i_2}-y_{i_1})^{2}(y_{i_4}-y_{i_1})^{2} \mid y_{i_1}\right]&=\bb{E}\left[(y_{i_2}-y_{i_1})^{2}\mid y_{i_1}\right]\bb{E}\left[(y_{i_2}-y_{i_1})^{2}\mid y_{i_1}\right] \\
&=\frac{1}{4}\Big(2\sigma^2+(\mu_1-y_{i_1})^2+(\mu_2-y_{i_1})^2\Big)^2.
\end{align*}
Again, we have 
\begin{align*}
y_{i_1}-\mu_1 \sim \frac{1}{2}\ca{N}(0,\sigma^2)+\frac{1}{2}\ca{N}(\mu_2-\mu_1,\sigma^2) \quad \text{and} \quad y_{i_1}-\mu_2 \sim \frac{1}{2}\ca{N}(0,\sigma^2)+\frac{1}{2}\ca{N}(\mu_1-\mu_2,\sigma^2).
\end{align*} 
Hence,
\begin{align*}
&\bb{E} \Big(2\sigma^2+(\mu_1-y_{i_1})^2+(\mu_2-y_{i_1})^2\Big)^2 = 4\sigma^4+\bb{E} (\mu_1-y_{i_1})^4+\bb{E} (\mu_1-y_{i_2})^4\\
&+4\sigma^2(\bb{E} (\mu_1-y_{i_1})^2+\bb{E} (\mu_1-y_{i_2})^2)+\bb{E}((\mu_1-y_{i_1})^2(\mu_2-y_{i_1})^2). 
\end{align*}
We have 
\begin{align*}
&\bb{E} (\mu_1-y_{i_1})^4=\bb{E} (\mu_2-y_{i_1})^4=3\sigma^4+\frac{(\mu_1-\mu_2)^4}{2}+3\sigma^2(\mu_1-\mu_2)^2\\
&\bb{E} (\mu_1-y_{i_1})^2=\bb{E} (\mu_2-y_{i_1})^2=\sigma^2+\frac{(\mu_1-\mu_2)^2}{2} \\
&\bb{E}((\mu_1-y_{i_1})^2(\mu_2-y_{i_1})^2) = \bb{E}((\mu_1-y_{i_1})^2(\mu_2-\mu_1+\mu_1-y_{i_1})^2) \\ 
&= \bb{E}[(\mu_1-y_{i_1})^4+(\mu_1-\mu_2)^2(\mu_1-y_{i_1})^2+2(\mu_2-\mu_1)(\mu_1-y_{i_1})^3] \\
&= 3\sigma^4+\frac{(\mu_1-\mu_2)^4}{2}+5\sigma^2(\mu_1-\mu_2)^2.
\end{align*}
Plugging in, we get
\begin{align*}
\bb{E} \Big(2\sigma^2+(\mu_1-y_{i_1})^2+(\mu_2-y_{i_1})^2\Big)^2=17\sigma^4+\frac{3(\mu_1-\mu_2)^4}{2}+13\sigma^2(\mu_1-\mu_2)^2.
\end{align*}
Hence, we obtain
\begin{align*}
\bb{E}\left[(y_{i_1}-y_{i_2})^{2}(y_{i_1}-y_{i_4})^{2}\right] \le 7\Big(\sigma^2+\frac{(\mu_1-\mu_2)^2}{2}\Big)^{2}.
\end{align*}
which proves the claim.
\end{proof}

From Claim \ref{claim:tech}, we can conclude that 
\begin{align*}
\bb{E}(S^j_{2,t})^{2} \le 12\Big(\sigma^2+\frac{(\mu_1-\mu_2)^2}{4}\Big)^{2}.
\end{align*}
From this point onwards, the analysis in this lemma is very similar to Lemma \ref{lem:mean}. We can use Chebychev's inequality to say that 
\begin{align*}
\Pr\Big(\left|S^j_{2,t}-{\rm var}X\right| \ge \epsilon_2 \Big) \le \frac{{\rm var}S^j_{2,t}}{t\epsilon_2^2} \le \frac{\bb{E}(S^j_{2,t})^{2}}{t\epsilon_2^2}. 
\end{align*} 
Therefore, we obtain by noting that $t \ge 1$ as well, 
\begin{align*}
\Pr\Big(\left|S^j_{2,t}-{\rm var}X\right| \ge \epsilon_2 \Big) \le \frac{12(\sigma^2+(\mu_1-\mu_2)^2/4)^{2}}{t\epsilon_2^2} \le \frac{1}{3} 
\end{align*} 
for $t=O(\Big\lceil(\sigma^2+(\mu_1-\mu_2)^2)^{2}/\epsilon_2^2 \Big\rceil)$. At this point, doing the same analysis as in Lemma \ref{lem:mean} shows that $B=36\log \eta^{-1}$ batches of batchsize $t$ is sufficient to estimate the variance within an additive error of $\epsilon_2$ with probability at least $1-2\eta$. Therefore the total sample complexity required is $  T=O(\log \eta^{-1}\Big\lceil(\sigma^2+(\mu_1-\mu_2)^2)^{2}/\epsilon_2^2 \Big\rceil)$ thus proving the lemma.

\end{proof}

\section{Proof of Lemma \ref{lem:test}}{\label{sec:test_proof}}
 Suppose, we use $O\Big(\Big\lceil\frac{1}{\epsilon^2} \Big\rceil\log \eta^{-1}\Big)$ samples to recover $\hat{\mu_1}$ and $\hat{\mu_2}$ using the method of moments. According to the guarantee provided in Theorem \ref{thm:mom}, we must have with probability at least $1-1/\eta$, 
\begin{align*}
\left|\hat{\mu}_i-\mu \right| \le 2(\epsilon+\sqrt{\epsilon})\sqrt{\sigma^2+(\mu_1-\mu_2)^2} \quad \text{for } i=1,2.
\end{align*}
Therefore, we have
\begin{align*}
& \left|\mu_1-\mu_2\right|-\left|\mu_1-\hat{\mu}_1 \right|-\left|\mu_2-\hat{\mu}_2 \right| \le \left| \hat{\mu}_1-\hat{\mu_2} \right| \le \left|\mu_1-\mu_2\right|+\left|\mu_1-\hat{\mu}_1 \right|+\left|\mu_2-\hat{\mu}_2 \right|  \\
 & \left| \left| \hat{\mu}_1-\hat{\mu_2}\right|-\left| \mu_1-\mu_2\right| \right| \le 4(\epsilon+\sqrt{\epsilon})\sqrt{\sigma^2+(\mu_1-\mu_2)^2}.    
\end{align*} 

We will substitute $\epsilon=1/256$. In that case we have 
\begin{align*}
\left| \left| \hat{\mu}_1-\hat{\mu_2}\right|-\left| \mu_1-\mu_2\right| \right| \le \frac{17}{64}\sqrt{\sigma^2+(\mu_1-\mu_2)^2} \le \frac{17\sigma}{64} +\frac{17\left|\mu_1-\mu_2\right|}{64}.
\end{align*}
and therefore
\begin{align*}
-\frac{17\sigma}{64} +\frac{47\left|\mu_1-\mu_2\right|}{64}  \le \left| \hat{\mu}_1-\hat{\mu_2}\right| \le  \frac{17\sigma}{64} +\frac{81\left|\mu_1-\mu_2\right|}{64}.
\end{align*}
Hence, we have
\begin{align*}
-\frac{17\sigma}{81} +\frac{64\left|\hat{\mu}_1-\hat{\mu}_2\right|}{81}  \le \left| \mu_1-\mu_2\right| \le  \frac{17\sigma}{47} +\frac{64\left|\hat{\mu}_1-\hat{\mu}_2\right|}{47}.
\end{align*}
\paragraph{Case 1 ($\sigma \ge \gamma$):}
This implies that if $|\hat{\mu}_1-\hat{\mu}_2| \le 15\sigma/32$, then $|\mu_1-\mu_2| \le \sigma$ and we will use the Method of Moments (Algorithm \ref{algo:mom}). On the other hand, if $|\hat{\mu}_1-\hat{\mu}_2| \ge 15\sigma/32$, then $|\mu_1-\mu_2| \ge 13\sigma/81$ and we will use EM algorithm (Algorithm \ref{algo:em}). The sample complexity required is $O(\Big\lceil\log \eta^{-1}\Big\rceil)$ samples.
\paragraph{Case 2 ($\sigma \le \gamma$):}
This implies that if $|\hat{\mu}_1-\hat{\mu}_2| \le 15\gamma/32$, then $|\mu_1-\mu_2| \le \gamma$ and we will fit a single gaussian (Algorithm \ref{algo:single}) to recover the means. On the other hand, if $|\hat{\mu}_1-\hat{\mu}_2| \ge 15\gamma/32$, then $|\mu_1-\mu_2| \ge 13\gamma/81$ and we will use EM algorithm (Algorithm \ref{algo:em}). The sample complexity required is $O(\Big\lceil\log \eta^{-1}\Big\rceil)$ samples.

\section{Proof of Lemma \ref{lem:align}}{\label{sec:align_proof}}

We have 
\begin{align*}
&\mu_{1,1} = \langle \f{x}^1,\beta_1 \rangle \quad \mu_{1,2} = \langle \f{x}^1,\beta_2 \rangle \quad \mu_{2,1} = \langle \f{x}^2,\beta_1 \rangle \quad \mu_{2,2} = \langle \f{x}^2,\beta_2 \rangle  \\
&\mu_{\s{sum},1} = \langle \f{x}^1+\f{x}^2,\beta_1 \rangle \quad \mu_{\s{sum},2} = \langle \f{x}^1+\f{x}^2,\beta_2 \rangle \quad \mu_{\s{diff},1} = \langle \f{x}^1-\f{x}^2,\beta_1 \rangle \quad \mu_{\s{diff},2} = \langle \f{x}^1-\f{x}^2,\beta_2 \rangle.
\end{align*}  
For a particular unknown mean $\mu_{\cdot,\cdot}$, we will denote the corresponding recovered estimate by $\hat{\mu}_{\cdot,\cdot}$ and moreover, let us assume without loss of generality that $\pi_1,\pi_2,\pi_{\s{sum}},\pi_{\s{diff}}$ are all same and the identity permutation itself (but note that this fact is unknown). If all the unknown parameters are recovered upto an additive error of $\gamma$, then we must have 
\begin{align*}
\left| \hat{\mu}_{\s{sum},1}-\hat{\mu}_{1,1}-\hat{\mu}_{2,1} \right| \le \left| \hat{\mu}_{\s{sum},1}- \mu_{\s{sum},1}\right|+\left| \mu_{1,1}-\hat{\mu}_{1,1} \right|+\left| \mu_{2,1}-\hat{\mu}_{2,1} \right| \le 3\gamma.
\end{align*}
\begin{align*}
\left| \hat{\mu}_{\s{diff},1}-\hat{\mu}_{1,1}+\hat{\mu}_{2,1} \right| \le \left| \hat{\mu}_{\s{diff},1}- \mu_{\s{diff},1}\right|+\left| \mu_{1,1}-\hat{\mu}_{1,1} \right|+\left| \mu_{2,1}-\hat{\mu}_{2,1} \right| \le 3\gamma.
\end{align*}
On the other hand, we must have 
\begin{align*}
\left| \hat{\mu}_{\s{sum},1}-\hat{\mu}_{1,1}-\hat{\mu}_{2,2} \right| & \ge \left| \hat{\mu}_{\s{sum},1}-\mu_{\s{sum},1}+\mu_{1,1}-\hat{\mu}_{1,1}+\mu_{2,1}-\mu_{2,2}+\mu_{2,2}- \hat{\mu}_{2,2} \right|  \\
& \ge \left|\mu_{2,1}-\mu_{2,2}\right|-\left| \hat{\mu}_{\s{sum},1}- \mu_{\s{sum},1}\right|-\left| \mu_{1,1}-\hat{\mu}_{1,1} \right|-\left| \mu_{2,1}-\hat{\mu}_{2,1} \right|   \\
& \ge \left|\mu_{2,1}-\mu_{2,2}\right|-3\gamma.
\end{align*}
\begin{align*}
\left| \hat{\mu}_{\s{diff},1}-\hat{\mu}_{1,1}+\hat{\mu}_{2,2} \right| & \ge \left| \hat{\mu}_{\s{diff},1}-\mu_{\s{diff},1}+\mu_{1,1}-\hat{\mu}_{1,1}-\mu_{2,1}+\mu_{2,2}-\mu_{2,2}+ \hat{\mu}_{2,2} \right|  \\
& \ge \left|\mu_{2,1}-\mu_{2,2}\right|-\left| \hat{\mu}_{\s{diff},1}- \mu_{\s{diff},1}\right|-\left| \mu_{1,1}-\hat{\mu}_{1,1} \right|-\left| \mu_{2,1}-\hat{\mu}_{2,1} \right|   \\
& \ge \left|\mu_{2,1}-\mu_{2,2}\right|-3\gamma.
\end{align*}
\begin{align*}
\left| \hat{\mu}_{\s{sum},1}-\hat{\mu}_{1,2}-\hat{\mu}_{2,1} \right| & \ge \left| \hat{\mu}_{\s{sum},1}-\mu_{\s{sum},1}+\mu_{2,1}-\hat{\mu}_{2,1}+\mu_{1,1}-\mu_{1,2}+\mu_{1,2}- \hat{\mu}_{1,2} \right|  \\
& \ge \left|\mu_{1,1}-\mu_{1,2}\right|-\left| \hat{\mu}_{\s{sum},1}- \mu_{\s{sum},1}\right|-\left| \mu_{2,1}-\hat{\mu}_{2,1} \right|-\left| \mu_{1,2}-\hat{\mu}_{1,2} \right|   \\
& \ge \left|\mu_{1,1}-\mu_{1,2}\right|-3\gamma
\end{align*}
\begin{align*}
\left| \hat{\mu}_{\s{diff},1}-\hat{\mu}_{1,2}+\hat{\mu}_{2,1} \right| & \ge \left| \hat{\mu}_{\s{diff},1}-\mu_{\s{diff},1}-\mu_{2,1}+\hat{\mu}_{2,1}+\mu_{1,1}-\mu_{1,2}+\mu_{1,2}- \hat{\mu}_{1,2} \right|  \\
& \ge \left|\mu_{1,1}-\mu_{1,2}\right|-\left| \hat{\mu}_{\s{diff},1}- \mu_{\s{diff},1}\right|-\left| \mu_{2,1}-\hat{\mu}_{2,1} \right|-\left| \mu_{1,2}-\hat{\mu}_{1,2} \right|   \\
& \ge \left|\mu_{1,1}-\mu_{1,2}\right|-3\gamma
\end{align*}
\begin{align*}
\left| \hat{\mu}_{\s{sum},1}-\hat{\mu}_{1,2}-\hat{\mu}_{2,2} \right| & \ge \left| \hat{\mu}_{\s{sum},1}-\mu_{\s{sum},1}+\mu_{1,1}+\mu_{2,1}-\mu_{1,2}-\mu_{2,2}+\mu_{1,2}- \hat{\mu}_{1,2}+\mu_{2,2}- \hat{\mu}_{2,2} \right|  \\
& \ge \left|\mu_{1,1}+\mu_{2,1}-\mu_{1,2}-\mu_{2,2}\right|-\left| \hat{\mu}_{\s{sum},1}- \mu_{\s{sum},1}\right|-\left| \mu_{1,2}-\hat{\mu}_{1,2} \right|-\left| \mu_{2,2}-\hat{\mu}_{2,2} \right|   \\
& \ge \left|\mu_{1,1}+\mu_{2,1}-\mu_{1,2}-\mu_{2,2}\right|-3\gamma.
\end{align*}
\begin{align*}
\left| \hat{\mu}_{\s{diff},1}-\hat{\mu}_{1,2}+\hat{\mu}_{2,2} \right| & \ge \left| \hat{\mu}_{\s{diff},1}-\mu_{\s{diff},1}+\mu_{1,1}-\mu_{2,1}-\mu_{1,2}+\mu_{2,2}+\mu_{1,2}- \hat{\mu}_{1,2}-\mu_{2,2}+ \hat{\mu}_{2,2} \right|  \\
& \ge \left|\mu_{1,1}-\mu_{2,1}-\mu_{1,2}+\mu_{2,2}\right|-\left| \hat{\mu}_{\s{sum},1}- \mu_{\s{sum},1}\right|-\left| \mu_{1,2}-\hat{\mu}_{1,2} \right|-\left| \mu_{2,2}-\hat{\mu}_{2,2} \right|   \\
& \ge \left|\mu_{1,1}-\mu_{2,1}-\mu_{1,2}+\mu_{2,2}\right|-3\gamma.
\end{align*}
Similar guarantees also exist for $\mu_{\s{sum},2}$ and $\mu_{\s{diff},2}$ and therefore we must have 
\begin{align*}
&\left| \hat{\mu}_{\s{sum},2}-\hat{\mu}_{1,2}-\hat{\mu}_{2,2} \right| \le 3\gamma  \\
&\left| \hat{\mu}_{\s{diff},2}-\hat{\mu}_{1,2}+\hat{\mu}_{2,2} \right| \le 3\gamma   \\
&\left| \hat{\mu}_{\s{sum},2}-\hat{\mu}_{1,2}-\hat{\mu}_{2,1} \right| \ge |\mu_{2,1}-\mu_{2,2}|-3\gamma  \\
&\left| \hat{\mu}_{\s{diff},2}-\hat{\mu}_{1,2}+\hat{\mu}_{2,1} \right| \ge |\mu_{2,1}-\mu_{2,2}|-3\gamma   \\ 
&\left| \hat{\mu}_{\s{sum},2}-\hat{\mu}_{1,1}-\hat{\mu}_{2,2} \right| \ge |\mu_{1,1}-\mu_{1,2}|-3\gamma  \\ 
&\left| \hat{\mu}_{\s{diff},2}-\hat{\mu}_{1,1}+\hat{\mu}_{2,2} \right| \ge |\mu_{1,1}-\mu_{1,2}|-3\gamma   \\
&\left| \hat{\mu}_{\s{sum},2}-\hat{\mu}_{1,1}-\hat{\mu}_{2,1} \right| \ge \left|\mu_{1,1}+\mu_{2,1}-\mu_{1,2}-\mu_{2,2}\right|-3\gamma \\
&\left| \hat{\mu}_{\s{diff},2}-\hat{\mu}_{1,1}-\hat{\mu}_{2,1} \right| \ge  \left|\mu_{1,1}-\mu_{2,1}-\mu_{1,2}+\mu_{2,2}\right|-3\gamma  
\end{align*}
Let us consider the case when $|\mu_{1,1}-\mu_{1,2}| \ge 9\gamma$ and $|\mu_{2,1}-\mu_{2,2}| \ge 9\gamma$.  We will have 
\begin{align*}
&\left| \hat{\mu}_{\s{sum},1}-\hat{\mu}_{1,1}-\hat{\mu}_{2,1} \right| \le 3\gamma \quad \left| \hat{\mu}_{\s{sum},1}-\hat{\mu}_{1,1}-\hat{\mu}_{2,2} \right| \ge 6\gamma \quad \left| \hat{\mu}_{\s{sum},1}-\hat{\mu}_{1,2}-\hat{\mu}_{2,1} \right| \ge 6\gamma \\
&\left| \hat{\mu}_{\s{diff},1}-\hat{\mu}_{1,1}+\hat{\mu}_{2,1} \right| \le 3\gamma \quad \left| \hat{\mu}_{\s{diff},1}-\hat{\mu}_{1,1}+\hat{\mu}_{2,2} \right| \ge 6\gamma \quad \left| \hat{\mu}_{\s{diff},1}-\hat{\mu}_{1,2}+\hat{\mu}_{2,1} \right| \ge 6\gamma \\
&\left| \hat{\mu}_{\s{sum},2}-\hat{\mu}_{1,2}-\hat{\mu}_{2,2} \right| \le 3\gamma \quad \left| \hat{\mu}_{\s{sum},2}-\hat{\mu}_{1,1}-\hat{\mu}_{2,2} \right| \ge 6\gamma \quad \left| \hat{\mu}_{\s{sum},1}-\hat{\mu}_{1,2}-\hat{\mu}_{2,1} \right| \ge 6\gamma \\
&\left| \hat{\mu}_{\s{diff},2}-\hat{\mu}_{1,2}+\hat{\mu}_{2,2} \right| \le 3\gamma \quad \left| \hat{\mu}_{\s{diff},1}-\hat{\mu}_{1,1}+\hat{\mu}_{2,2} \right| \ge 6\gamma \quad \left| \hat{\mu}_{\s{diff},1}-\hat{\mu}_{1,2}+\hat{\mu}_{2,1} \right| \ge 6\gamma \\
\end{align*}
Moreover, at least one of the following two must be true:
\begin{align*}
&\left| \hat{\mu}_{\s{sum},1}-\hat{\mu}_{1,2}-\hat{\mu}_{2,2} \right| \ge  15\gamma \quad \text{and} \quad \left| \hat{\mu}_{\s{sum},2}-\hat{\mu}_{1,1}-\hat{\mu}_{2,1} \right| \ge 15\gamma \\
& \text{or} \quad \left| \hat{\mu}_{\s{diff},1}-\hat{\mu}_{1,2}+\hat{\mu}_{2,2} \right| \ge 15\gamma \quad \text{and} \quad \left| \hat{\mu}_{\s{diff},2}-\hat{\mu}_{1,1}+\hat{\mu}_{2,1} \right| \ge 15\gamma
\end{align*}
depending on whether $\mu_{1,1}-\mu_{1,2}$ and $\mu_{2,1}-\mu_{2,2}$ have the same sign or not. This shows that either for the sum query or for the difference query, only the correct set of means is closest to their corresponding value (sum or difference) and any wrong choice of means is away from that particular value (sum or difference). Hence the lemma is proved.

\section{Proof of Lemma \ref{lem:exist} and \ref{lem:prob}}{\label{sec:exist}}
\begin{proof}[Proof of Lemma \ref{lem:exist}]
Notice that for a particular query $\f{x}^i, i\in [m]$, the difference of the means $\mu_{i,1}-\mu_{i,2}$ is distributed according to
\begin{align*}
\mu_{i,1}-\mu_{i,2} \sim \ca{N}(0,||\f{\beta}^1-\f{\beta}^2||_2^2).
\end{align*}
Therefore, we have 
\begin{align}
\Pr(|\mu_{i,1}-\mu_{i,2}| \le 13\gamma) = \int_{-13\gamma}^{13\gamma}\frac{e^{-x^2/2||\f{\beta}^1-\f{\beta}^2||_2^2}}{\sqrt{2\pi ||\f{\beta}^1-\f{\beta}^2||_2^2}}dx  \le \frac{13\sqrt{2}\gamma}{\sqrt{\pi}||\f{\beta}^1-\f{\beta}^2||_2}.
\end{align}
where the upper bound is obtained by using $e^{-x^2/2||\f{\beta}^1-\f{\beta}^2||_2^2} \le 1$. Therefore the probability that for all the $m'$ queries $\{\f{x}^i\}_{i=1}^{m}$, the difference between the means is less than $13\gamma$ must be
\begin{align*}
\Pr(\bigcup_{i=1}^{m'} \mu_{i,1}-\mu_{i,2} \le 13\gamma)=\prod_{i=1}^{m'}\Pr(\mu_{i,1}-\mu_{i,2} \le 13\gamma) \le \Big(\frac{13\sqrt{2}\gamma}{\sqrt{\pi}||\f{\beta}^1-\f{\beta}^2||_2}\Big)^{m'} = e^{-m'\log \frac{\sqrt{\pi}||\f{\beta}^1-\f{\beta}^2||_2}{13\sqrt{2}\gamma}}
\end{align*}
Therefore for $m'=\Big\lceil\log \eta^{-1}/\log \frac{\sqrt{\pi}||\f{\beta}^1-\f{\beta}^2||_2}{13\sqrt{2}\gamma}\Big\rceil$, we have that 
$
\Pr(\bigcup_{i=1}^{m'} \mu_{i,1}-\mu_{i,2} \le 13\gamma) \le \eta.
$
\end{proof}

\begin{proof}[Proof of Lemma \ref{lem:prob}]
The proof of Lemma \ref{lem:prob} is very similar to the proof of Lemma \ref{lem:exist}. Again for a particular query $\f{x}^i, i\in [m]$, the difference of the means $\mu_{i,1}-\mu_{i,2}$ is distributed according to
\begin{align*}
\mu_{i,1}-\mu_{i,2} \sim \ca{N}(0,||\f{\beta}^1-\f{\beta}^2||_2^2).
\end{align*}
Therefore, we have for any constant $c_1>0$,
\begin{align*}
\Pr(|\mu_{i,1}-\mu_{i,2}| \le c_1 \sigma) = \int_{-c_1 \sigma}^{c_1 \sigma}\frac{e^{-x^2/2||\f{\beta}^1-\f{\beta}^2||_2^2}}{\sqrt{2\pi ||\f{\beta}^1-\f{\beta}^2||_2^2}}dx  \le \frac{\sqrt{2}c_1 \sigma}{\sqrt{\pi}||\f{\beta}^1-\f{\beta}^2||_2}.
\end{align*}
where the upper bound is obtained by using $e^{-x^2/2||\f{\beta}^1-\f{\beta}^2||_2^2} \le 1$. 
Hence the lemma is proved by substituting $c_2=\sqrt{2}c_1/\sqrt{\pi}$.
\end{proof}

\section{Proof of Theorem \ref{thm:main} ($\gamma =\Omega (\left|\left|\beta^1-\beta^2\right|\right|_2)$)}\label{sec:thm1_rem}.

\begin{algorithm}[htbp]
\caption{ \textsc{Recover unknown vectors $2$($\sigma,\gamma$)} Recover the unknown vectors  $\f{\beta}^1$ and $\f{\beta}^2$ \label{algo:main2}}
\begin{algorithmic}[1]
\STATE Set $m=c_s k \log n$ and $T=O\Big(\Big\lceil\sigma^2\log k/(\gamma-0.8\left|\left|\beta^1-\beta^2\right|\right|_2)^2 \Big\rceil \Big)$.
\STATE Sample $\f{x}^1,\f{x}^2,\dots,\f{x}^m \sim \ca{N}(\f{0},\f{I}_n)$ independently.
\FOR{$i=1,2,\dots,m$}
         \STATE Compute $\hat{\mu}_{i}$ by running Algorithm \textsc{Fit a single gaussian }($\f{x}^i,T$).
\ENDFOR
\STATE Set $\f{u}$ to be the $m$-dimensional vector whose $i^{\s{th}}$ element is $\hat{\mu}_i$.
\STATE Set $\f{A}$ to be the $m \times n$ matrix such that its $i^{\s{th}}$ row is $\f{x}^i$, with each entry normalized by $\sqrt{m}$.
\STATE Set $\hat{\f{\beta}}$ to be the solution of the optimization problem $\min_{\f{z} \in \bb{R}^n} ||\f{z}||_1 \; \text{s.t.} \; ||\f{A}\f{z}-\frac{1}{\sqrt{m}}\f{u}||_2 \le \gamma$
\STATE Return $\hat{\f{\beta}}$.
\end{algorithmic}
\end{algorithm}

We will first assume $\gamma >0.8\left|\left| \beta^1-\beta^2 \right|\right|_2$ to prove the claim, and later extend this to any $\gamma =\Omega (\left|\left|\beta^1-\beta^2\right|\right|_2).$
The recovery procedure in the setting when $\gamma >0.8\left|\left| \beta^1-\beta^2 \right|\right|_2$ is described in Algorithm \ref{algo:main2}. We will start by proving the following claim
\begin{claim}
Algorithm \ref{algo:main2} returns a vector $\f{u}$ of length $m$ using $O\Big(m\Big\lceil\sigma^2\log \eta^{-1}/\epsilon^2 \Big\rceil\Big)$ queries such that 
\begin{align*}
    &\left|\f{u}[i]-\langle \f{x}^i,\beta^{1} \rangle \right| \le \epsilon+\frac{\left|\langle \f{x}^i,\beta^1-\beta^2\rangle\right|}{2} \\  
    &\left|\f{u}[i]-\langle \f{x}^i,\beta^{2} \rangle \right| \le
    \epsilon+\frac{\left|\langle \f{x}^i,\beta^1-\beta^2\rangle\right|}{2}
\end{align*}
 for all $i \in [m]$ with probability at least $1-m\eta$. 
\end{claim}
\begin{proof}
In Algorithm \ref{algo:main2}, for each query $\f{x}^i \sim \ca{N}(\f{0},\f{I}_n)$, we can use a batchsize of $O\Big(\Big\lceil\sigma^2\log \eta^{-1}/\epsilon^2 \Big\rceil\Big)$ to recover $\hat{\mu}_i$ such that 
\begin{align*}
    \left|\hat{\mu}_i-\frac{\langle \f{x}^i,\beta^1 \rangle+\langle \f{x}^i,\beta^2 \rangle}{2}\right|\le \epsilon
\end{align*}
with probability at least $1-\eta$ according to the guarantees of Lemma \ref{thm:mean2}. We therefore have
\begin{align*}
    &\left|\hat{\mu}_i-\langle \f{x}^i,\beta^1 \rangle\right| \le \left|\hat{\mu}^i-\frac{\langle \f{x}^i,\beta^1 \rangle 
    +\langle \f{x}^i,\beta^2 \rangle}{2}\right|\\
    &+\frac{\left|\langle \f{x}^i,\beta^1-\beta^2\rangle\right|}{2}\le \epsilon+\frac{\left|\langle \f{x}^i,\beta^1-\beta^2\rangle\right|}{2}.
\end{align*}
where the last inequality follows by using the guarantees on $\hat{\mu}_i$. We can show a similar chain of inequalities for $\left|\hat{\mu}_i-\langle \f{x}^i,\beta^2 \rangle\right|$ and finally take a union bound over all $i \in [m]$ to conclude the proof of the claim.
\end{proof}
Next, let us define the random variable $\omega_i \triangleq \left|\langle \f{x}^i,\beta^1-\beta^2 \rangle\right|$ where the randomness is over $\f{x}^i$. Subsequently let us define the $m$-dimensional vector $\f{b}$ whose $i^{\s{th}}$ element is $\epsilon+\omega_i/2$. Again, for $m \ge c_s k\log n$, let $\f{A}$ denote the matrix whose $i^{\s{th}}$ row is $\f{x}^i$ normalized by $\sqrt{m}$.
\begin{claim}
We must have
\begin{align*}
\left|\left|\f{A}\f{\beta}^{1}-\frac{\f{u}}{\sqrt{m}}\right|\right|_2 \le \frac{\left|\left|\f{b}\right|\right|_2}{\sqrt{m}} \; \& \; 
\left|\left|\f{A}\f{\beta}^{2}-\frac{\f{u}}{\sqrt{m}}\right|\right|_2 \le \frac{\left|\left|\f{b}\right|\right|_2}{\sqrt{m}}.
\end{align*}
\end{claim}
\begin{proof}
The proof of the claim is immediate from definition of $\f{A}$ and $\f{b}$.
\end{proof}
Next, we show high probability bounds on $\ell_2$-norm of the vector $\f{b}$ in the following claim.
\begin{claim}{\label{claim:conc}}
We must have $\frac{\left|\left|\f{b}\right|\right|_2^2}{m} \le 2\epsilon^2+0.64\left|\left|\beta^1-\beta^2\right|\right|_2^2$ with probability at least $1-O(e^{-m})$. 
\end{claim}
\begin{proof}
Notice that
\begin{align*}
    \frac{\left|\left|\f{b}\right|\right|_2^2}{m} \le \frac{1}{m}\sum_{i=1}^{m}\Big(\epsilon+\frac{\omega_i}{2}\Big)^2\le \frac{2}{m}\sum_{i=1}^{m}\Big(\epsilon^2+\frac{\omega_i^2}{4}\Big).
\end{align*}
Using the fact that $\f{x}^i \sim \ca{N}(\f{0},\f{I}_n)$ and by definition, we must have that $\omega_i$ is a random variable distributed according to $\ca{N}(0,\left|\left|\beta^1-\beta^2\right|\right|_2)$. Therefore, we have (see Lemma 1.2 \cite{boche2015survey})
\begin{align*}
    &\Pr\Bigg(\left|\sum_{i=1}^{m}\omega_i^2 - m\left|\left|\beta^1-\beta^2\right|\right|_2^2\right| \ge m\rho\left|\left|\beta^1-\beta^2\right|\right|_2^2\Bigg) \\
    &\le 2\exp\Big(-\frac{m}{2}\Big(\frac{\rho^2}{2}-\frac{\rho^3}{3}\Big)\Big)
\end{align*}
for $0<\rho<1$. Therefore, by substituting $\rho=0.28$, we get that 
\begin{align*}
    \frac{2}{m}\sum_{i=1}^{m}\Big(\epsilon^2+\frac{\omega_i^2}{4}\Big) \le 2\epsilon^2+0.64\left|\left|\beta^1-\beta^2\right|\right|_2^2
\end{align*}
with probability at least $1-O(e^{-m})$.
\end{proof}
From Claim \ref{claim:conc}, we get 
\begin{align*}
 \frac{\left|\left|\f{b}\right|\right|_2}{\sqrt{m}}  \le \sqrt{2}\epsilon+0.8\left|\left|\beta^1-\beta^2\right|\right|_2.   
\end{align*}
where we use the inequality $\sqrt{a+b}\le \sqrt{a}+\sqrt{b}$ for $a,b>0$. Subsequently we solve the following convex optimization problem
\begin{align*}
\min_{\f{z} \in \bb{R}^n} ||\f{z}||_1 \; \text{s.t.} \; ||\f{A}\f{z}-\frac{\f{u}}{\sqrt{m}}||_2 \le  \gamma
\end{align*}
where $\gamma=\sqrt{2}\epsilon+0.8\left|\left|\beta^1-\beta^2\right|\right|_2$
in order to recover $\hat{\beta}$ and return it as estimate of both $\beta^1,\beta^2$. For $m=O(k\log n), \eta = (m\log n)^{-1}$ and $\sqrt{2}\epsilon=\gamma-0.8\left|\left|\beta^1-\beta^2\right|\right|_2$, the number of queries required is $O\Big(k\log n\Big\lceil \sigma^2\log k/(\gamma-0.8\left|\left|\beta^1-\beta^2\right|\right|_2)^2\Big\rceil \Big)$. Further, by using the theoretical guarantees provided in Theorem 1.6 in \cite{boche2015survey}, we obtain the guarantees of the main theorem with error probability atmost $o(1)$. Again, by substituting the definition of the Noise Factor $\s{NF}=\gamma/\sigma$ and the Signal to Noise ratio
$\s{SNR}=O(\left|\left|\beta^1-\beta^2\right|\right|_2^2/\sigma^2)$, we obtain the query complexity to be
$$O\Big(k\log n\Big\lceil \frac{\log k}{(\s{NF}-0.8\sqrt{\s{SNR}})^2}\Big\rceil\Big).$$
Now let us assume any $\gamma = \Omega(\left|\left|\beta^1-\beta^2\right|\right|_2).$ If the desired $\gamma < \left|\left|\beta^1-\beta^2\right|\right|_2$, then one can just define $\gamma'= \left|\left|\beta^1-\beta^2\right|\right|_2$ and obtain a precision $\gamma'$ which is a constant factor within $\gamma$. Further, the query complexity also becomes independent of the noise factor since $\s{NF}=\sqrt{\s{SNR}}$ for this choice of $\gamma'$ and thus we obtain the promised query complexity in Theorem \ref{thm:main}.

\section{Discussion on Noiseless Setting $\sigma=0$}\label{sec:discuss}
\paragraph{Step 1:} In the noiseless setting, we obtain $m=O(k\log n)$ query vectors $\f{x}^1,\f{x}^2,\dots,\f{x}^m$ sampled i.i.d according to $\ca{N}(\f{0},\f{I}_n)$ and repeat each of them for $2\log m$ times. For a particular query $\f{x}_i$, the probability that we do not obtain any samples from $\f{\beta}^1$ or $\f{\beta}^2$ is at most $(1/2)^{2\log m}$. We can take a union bound  to conclude that for all queries, we obtain samples from both $\f{\beta}^1$ and $\f{\beta}^2$ with probability at least $1-O(m^{-1})$. Further note that for each query $\f{x}^i$, $\langle \f{x}^i,\beta^1-\beta^2 \rangle$ is distributed according to $\ca{N}(0,\left|\left|\beta^1-\beta^2\right|\right|_2^2)$ and therefore, it must happen with probability $1$ that $\langle \f{x}^i,\beta^1 \rangle \neq \langle \f{x}^i,\beta^1 \rangle$. Thus for each query $\f{x}_i$, we can recover the tuple $(\langle \f{x}^i,\beta^1 \rangle,\langle \f{x}^i,\beta^2 \rangle)$ but we cannot recover the ordering i.e. we do not know which element of the tuple corresponds to $\beta^1$ and which one to $\beta^2$.

\paragraph{Step 2:} Note that we are still left with the alignment step where we need to cluster the $2m$ recovered parameters $\{(\langle \f{x}^i,\beta^1 \rangle,\langle \f{x}^i,\beta^2 \rangle)\}_{i=1}^{m}$ into two clusters of size $m$ each so that there exists exactly one element from each tuple in each of the two clusters and all the elements in the same cluster correspond to the same unknown vector. In order to complete this step, we use ideas from \cite{KrishnamurthyM019}. We query $\f{x}_1+\f{x}_i$  and $\f{x}_1-\f{x}_i$ for all $i \neq 1$ each for $2\log m$ times to the oracle and recover the tuples $(\langle \f{x}^1+\f{x}^i,\beta^1 \rangle,\langle \f{x}^1+\f{x}^i,\beta^2 \rangle)$ and $(\langle \f{x}^1-\f{x}^i,\beta^1 \rangle,\langle \f{x}^1-\f{x}^i,\beta^2 \rangle)$ for all $i \neq 1$. For a particular $i \in [m]\setminus\{1\}$, we will choose two elements (say $a$ and $b$) from the pairs $(\langle \mathbf{x}_1,\beta^1 \rangle,\langle \mathbf{x}_1, \beta^2 \rangle)$ and $(\langle \mathbf{x}_i,\beta^1 \rangle,\langle \mathbf{x}_i, \beta^2 \rangle)$ (one element from each pair) such that their sum belongs to the pair $\langle \mathbf{x}_1+\mathbf{x}_i,\beta^1 \rangle,\langle \mathbf{x}_1+\mathbf{x}_i,\beta^2 \rangle$ and their difference belongs to the pair $\langle \mathbf{x}_1-\mathbf{x}_i,\beta^1 \rangle,\langle \mathbf{x}_1-\mathbf{x}_i,\beta^2 \rangle$. In our algorithm, we will put $a,b$ into the same cluster and the other two elements into the other cluster.
From construction, we must put $(\langle \mathbf{x}_1,\beta^1 \rangle, \langle \mathbf{x}_i,\beta^1 \rangle)$ in one cluster and $(\langle \mathbf{x}_1,\beta^2 \rangle, \langle \mathbf{x}_i,\beta^2 \rangle)$
in other. Note that a failure in this step is not possible because the $2m$ recovered  parameters are different from each other with probability $1$.

\paragraph{Step 3:} Once we have clustered the samples, we have reduced our problem to the usual compressed sensing setting (with only $1$ unknown vector) and therefore we can run the well known convex optimization routine in order to recover the unknown vectors $\beta^1$ and $\beta^2$. The total query complexity is $O(k\log n\log k)$.

\section{`Proof of Concept' Simulations}
The methods of parameter recovery in Gaussian mixtures are compared in Fig~\ref{fig:corrinfvaryq}. As claimed in Sec.~\ref{sec:GMM}, the EM starts performing better than the method of moments when the gap between the parameters is large.

We have also run Algorithm \ref{algo:main} for different set of pairs of sparse vectors and example recovery results for visualization are shown in Figures~\ref{fig:varyn} and \ref{fig:varyk}. Note that, while quite accurate reconstruction is possible the vectors are not reconstructed in order, as to be expected.
\begin{figure*}[htbp]
  \begin{subfigure}[htbp]{\textwidth}
    \centering 
    \includegraphics[scale = 0.6]{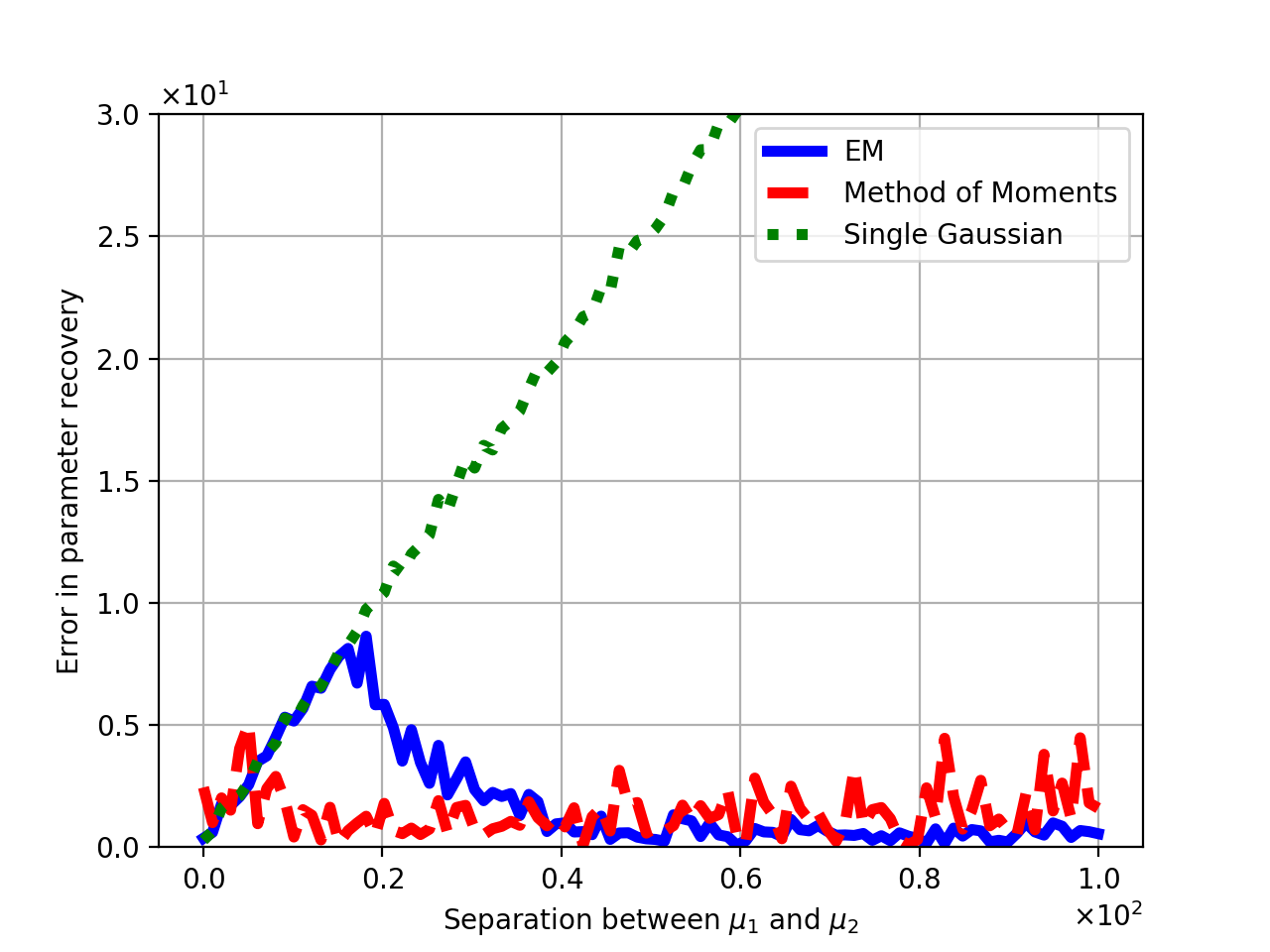}
    \caption{Comparison of the three techniques for recovery of parameters of a Gaussian mixture with $1000$ samples (see Algorithms \ref{algo:em},\ref{algo:mom} and \ref{algo:single}). The error in parameter recovery is plotted with separation between $\mu_1$ and $\mu_2$ (by keeping $\mu_1$ fixed at $0$ and varying $\mu_2$).}
          ~\label{fig:corrinfvaryq}
  \end{subfigure}
  \hfill
  \begin{subfigure}[htbp]{0.48\textwidth}
     \includegraphics[scale =  0.5]{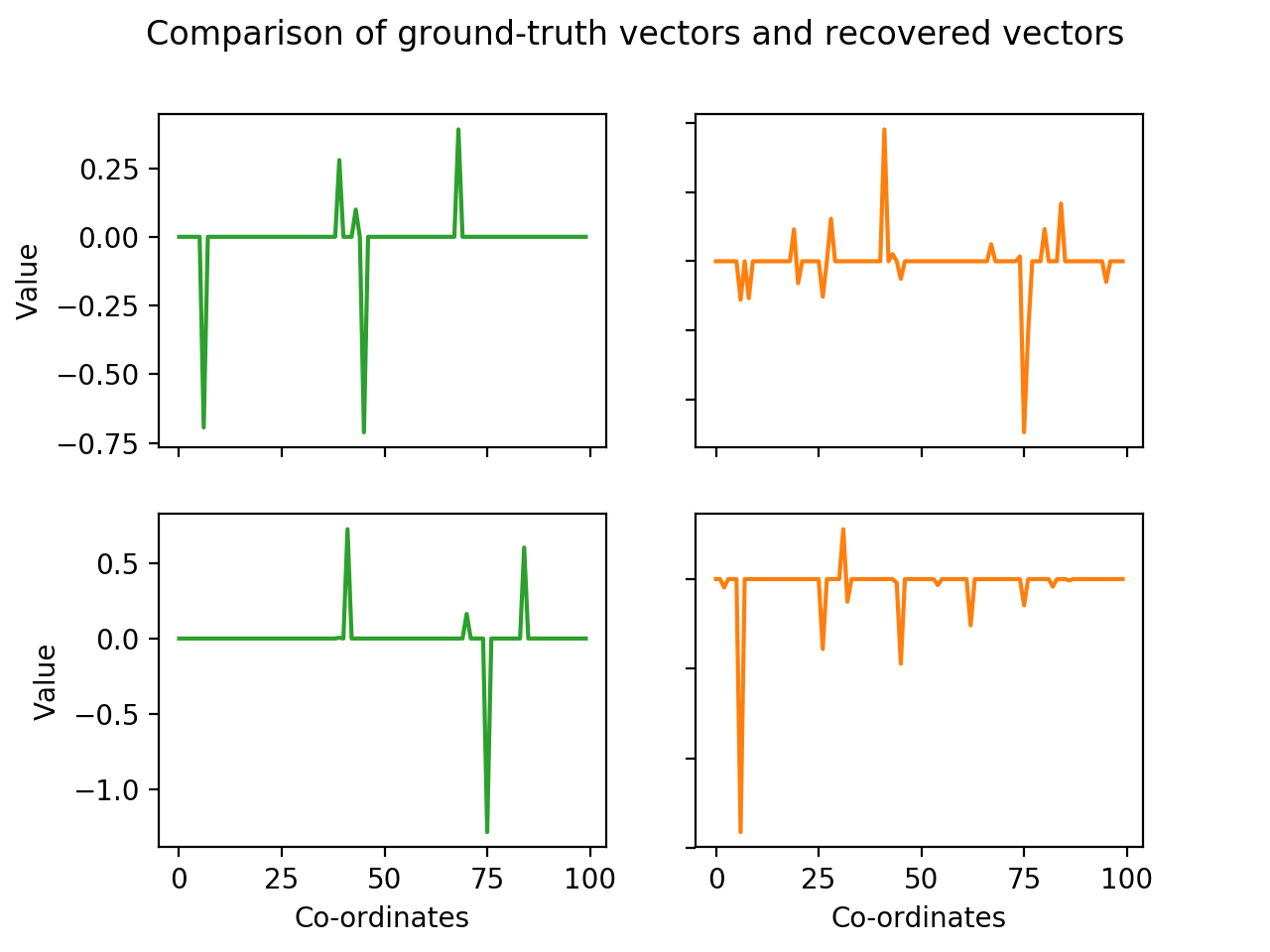}
     \caption{\small The $100$-dimensional ground truth vectors $\beta^1$ and $\beta^2$ with sparsity $k=5$ plotted in green (left) and the recovered vectors (using Algorithm \ref{algo:main}) $\hat{\beta}^1$ and $\hat{\beta}^2$ plotted in orange (right) using a batch-size $\sim 100$ for each of $150$ random gaussian queries. The order of the recovered vectors and the ground truth vectors is reversed.}
          ~\label{fig:varyn}
  \end{subfigure}\hfill 
\begin{subfigure}[htbp]{0.48 \textwidth}
   \includegraphics[scale = 0.5]{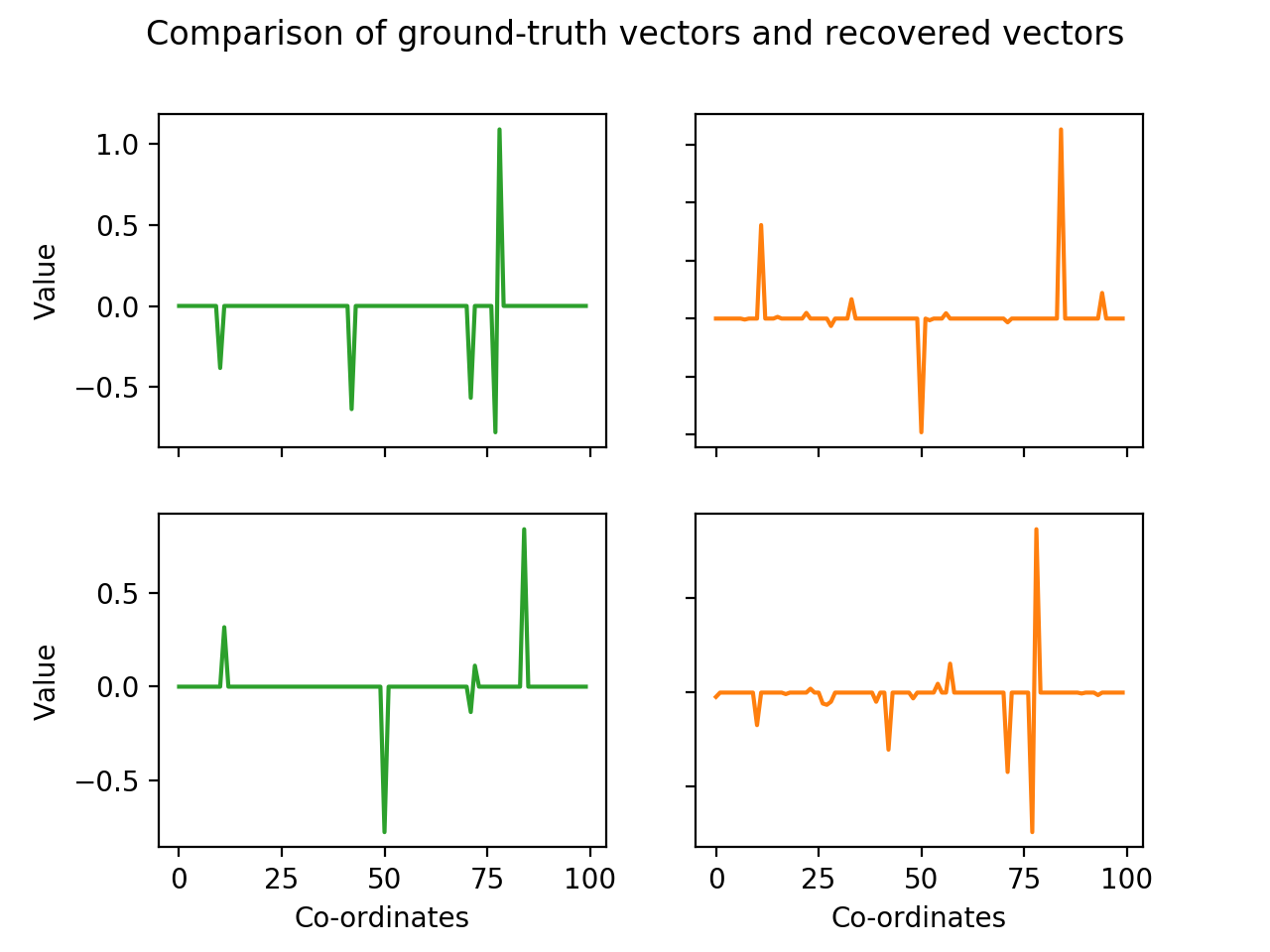}
   \caption{\small The $100$-dimensional ground truth vectors $\beta^1$ and $\beta^2$ with sparsity $k=5$ plotted in green (left) and the recovered vectors (using Algorithm \ref{algo:main}) $\hat{\beta}^1$ and $\hat{\beta}^2$ plotted in orange (right) using a batch-size $\sim 600$ for each of $150$ random gaussian queries. The order of the recovered vectors and the ground truth vectors is reversed.}
       ~\label{fig:varyk}
 \end{subfigure}%

\hfill

 \caption{\small Simulation results of our techniques.}
\end{figure*}

\end{document}